\documentclass[educ,sglanonrev]{informs-educ}
\RequirePackage{tgtermes}
\RequirePackage{newtxtext}
\RequirePackage{newtxmath}
\RequirePackage{bm}
\RequirePackage{endnotes}

\makeatletter

\def\theARTICLETOP{}

\def\theARTICLEABSTRACT{%
	\HOOKb
	\vspace*{18pt}
	\noindent
	\begin{minipage}[t]{\textwidth}\parindent1em
		\ABSfont
		\noindent\theABSTRACT\endgraf
		\vskip5pt
		\theFUNDING
		\theKEYWORDS
		\theSUBJECTCLASS
		\theAREAOFREVIEW
		\theMSCCLASS
		\theORMSCLASS
		\if@BLINDREV\else\theHISTORY\fi
		\noindent\hrulefill
	\end{minipage}%
	\vspace*{0pt}
}

\makeatother

\RRHSecondLine{\fs.7.9.\rm {\it\theJOURNAL}}
\LRHSecondLine{\fs.7.9.\rm {\it\theJOURNAL}}

\OneAndAHalfSpacedXII 

\usepackage{algorithm}
\usepackage{algpseudocode}
\usepackage{tikz}
\usepackage{hyperref}
\usepackage{enumitem}
\usepackage{booktabs,tabularx,array}

\usepackage[sort&compress]{natbib}
\bibpunct[, ]{[}{]}{,}{n}{}{,}%
\def\bibfont{\small}%

\usepackage{pifont}
\newcommand{\cmark}{\ding{51}}
\newcommand{\xmark}{\ding{55}}

\EquationsNumberedThrough    

\TheoremsNumberedThrough     
\ECRepeatTheorems  %

\newcommand{\bE}{\mathbb{E}}

\newcommand{\bbR}{\mathbb{R}}

\newcommand{\cH}{\mathcal{H}}
\newcommand{\cX}{\mathcal{X}}
\newcommand{\cD}{\mathcal{D}}

\newcommand{\wspo}{W_{\textnormal{SPO}}}
\newcommand{\wdfo}{W^{\textnormal{O}}_{\textnormal{DF}}}
\newcommand{\wdfr}{W^{\textnormal{R}}_{\textnormal{DF}}}
\newcommand{\ld}{\ell_{\mathrm{decision}}}

\MANUSCRIPTNO{EDUC-0001-2026.00}

\begin{document}
	
	
	\RUNAUTHOR{Mo Liu}
	
	\RUNTITLE{Tutorial for Decision-Focused Learning}
	
	\TITLE{Decision-Focused Learning: When and Why Traditional Prediction Models Fail}
	
	\ARTICLEAUTHORS{%
			\AUTHOR{Mo Liu}
			\AFF{Department of Statistics and Operations Research,
				University of North Carolina at Chapel Hill, \EMAIL{mo\_liu@unc.edu}}
			
		} 
		
		\ABSTRACT{%
			Plugging predictions of unknown parameters into downstream optimization problems, often referred to as the ``predict-then-optimize'' paradigm, has long been a standard approach in decision-making under uncertainty. However, improved predictive accuracy does not, in general, translate into improved decision quality. This disconnect has motivated growing interest in decision-focused learning (DFL) within the operations research community.
			This tutorial reviews recent developments in DFL and highlights key methodological insights, with a particular focus on stochastic linear programming as the downstream decision-making problem. We discuss why several widely used tools in traditional statistical learning are not directly suited to decision-focused settings and must be rethought, including (i) data collection strategies driven purely by predictive uncertainty and (ii) distributional distance measures such as the Wasserstein distance. We summarize properties of DFL that distinguish it from conventional predictive modeling and provide insights into the development of new decision-focused tools.
		}%
		
		
		
		
		
		\maketitle
		
		
		
		\section{Introduction}\label{sec:Intro}
		Decision-making typically involves solving an optimization problem. When the optimization problem contains unknown parameters, decision-making under uncertainty generally follows a two-stage pipeline. First, a statistical model is constructed to estimate the unknown quantities, which serve as parameters of the downstream optimization problem. Second, these estimated parameters are plugged into the optimization model to obtain a decision.
		In the first stage, when contextual data or observable features are available to help predict the uncertainty, the statistical model is typically a predictive model. This ``predict-then-optimize'' paradigm arises in a wide range of OR applications, including routing, inventory control, recommendation, pricing, matching, and healthcare operations.
		Although it is intuitive to expect smaller prediction errors to translate into higher-quality decisions, this intuition does not hold in general, either empirically or theoretically. This mismatch between prediction error and decision quality has motivated \textit{decision-focused learning} (DFL), which incorporates the structure of the downstream decision-making problem into the construction of the prediction model.
		
		This tutorial focuses on the geometric and statistical ideas underlying the mismatch between prediction error and decision loss, and illustrates how to leverage this mismatch to study statistical learning problems, such as data collection and uncertainty quantification. The tutorial is intended for PhD students and researchers who are familiar with either linear programming or statistical learning, and who are interested in research at their intersection. By the end of the tutorial, readers will understand key statistical learning challenges in DFL, gain high-level geometric intuition for analyzing DFL problems, and be introduced to a range of promising research directions.

		Throughout the tutorial, we use the notation summarized in Table~\ref{tab:notation}. Specifically, $x \in \bbR^k$ denotes the observable feature vector, $c \in \bbR^{d_1}$ denotes the unknown coefficient vector entering the downstream optimization problem, $h(\cdot): \bbR^k \to \bbR^{d_1}$ denotes a prediction model, and $w \in \bbR^{d_2}$ denotes the decision vector. We use $w^*(c)$ to denote the optimal decision under the realized parameter vector $c$, where $w^*(\cdot): \bbR^{d_1} \to \bbR^{d_2}$. The exact form of $w^*(c)$ depends on the downstream decision-making objective, which will be specified later. Some symbols are introduced in later sections, but we include them here in Table \ref{tab:notation} for reference.
		
		\begin{table}[htbp]
			\centering
			\caption{Summary of notation}
			\label{tab:notation}
			\small
			\renewcommand{\arraystretch}{1.12}
			\begin{tabular}{p{0.2\linewidth}p{0.76\linewidth}}
				\hline
				\textbf{Notation} & \textbf{Meaning} \\
				\hline
				$x\in\cX\subseteq\bbR^k$ & Observable feature vector or contextual information. \\
				$c\in\bbR^{d_1}$ & Realized uncertain coefficient vector entering the downstream optimization problem, such as a cost vector or a demand vector. \\
				$\hat c=h(x)$ & Prediction of the uncertain coefficient vector produced by model $h$. \\
				$\cH$ & Hypothesis class of prediction models. \\
				$w\in\bbR^{d_2}$ & Downstream decision vector.  \\
				$S\subseteq\bbR^{d_2}$ & Known bounded feasible region of the downstream optimization problem. \\
				$w^*(c)$ & Optimal decision induced by a plug-in coefficient vector $c$, with a fixed tie-breaking rule when multiple optimizers exist. \\
				$z(c)$ & Optimal value of the linear problem, $z(c):=\min_{w\in S}c^\top w$. \\
				$\ld(\hat c,c)$ & Per-instance decision loss, or SPO loss, defined as $c^\top w^*(\hat c)-c^\top w^*(c)$. \\
				$R_\ell(h)$, $R(h)$ & Population risk under a generic loss $\ell$; $R(h)$ denotes the decision risk under $\ld$. \\
				$\cD$ & Data-generating distribution of $(x,c)$. \\
				$w_{(j)}$, $N_j$ & Extreme points of $S$ and their associated normal cones in the cost space. \\
				$\operatorname{diam}(S)$ & Diameter of the feasible region under the chosen norm. \\
				\hline
			\end{tabular}
		\end{table}

		\subsection{Examples where traditional statistical learning fails}

		Consider a shortest path problem with two routes in Figure \ref{example:shortest}: one route has a fixed cost $1$, whereas the other has a feature-dependent random cost $Y$. The route choice changes only when the expected cost of the uncertain route crosses the threshold $1$. The collected data pairs consisting of the cost $Y$ and the feature variable $X$ (e.g., weather and travel conditions) are shown in the right panel of Figure~\ref{example:shortest}. 
		
		\begin{figure}[ht]
			\centering
			
			\begin{minipage}{0.3\textwidth}
				\centering
				\begin{tikzpicture}[>=stealth, thick]
					
					\node[circle, draw, minimum size=1cm] (L) at (0,0) {};
					\node[circle, draw, minimum size=1cm] (R) at (4,0) {};
					
					\draw[->] (0.5,0.6) 
					.. controls (2,1) .. 
					node[midway, below] {1} 
					(3.5,0.6);
					
					\draw[->] (0.5,-0.6) 
					.. controls (2,-1) .. 
					node[midway, below] {$Y=?$} 
					(3.5,-0.6);
					
				\end{tikzpicture}
			\end{minipage}
			\hfill
			\begin{minipage}{0.68\textwidth}
				\centering
				\includegraphics[width=\linewidth]{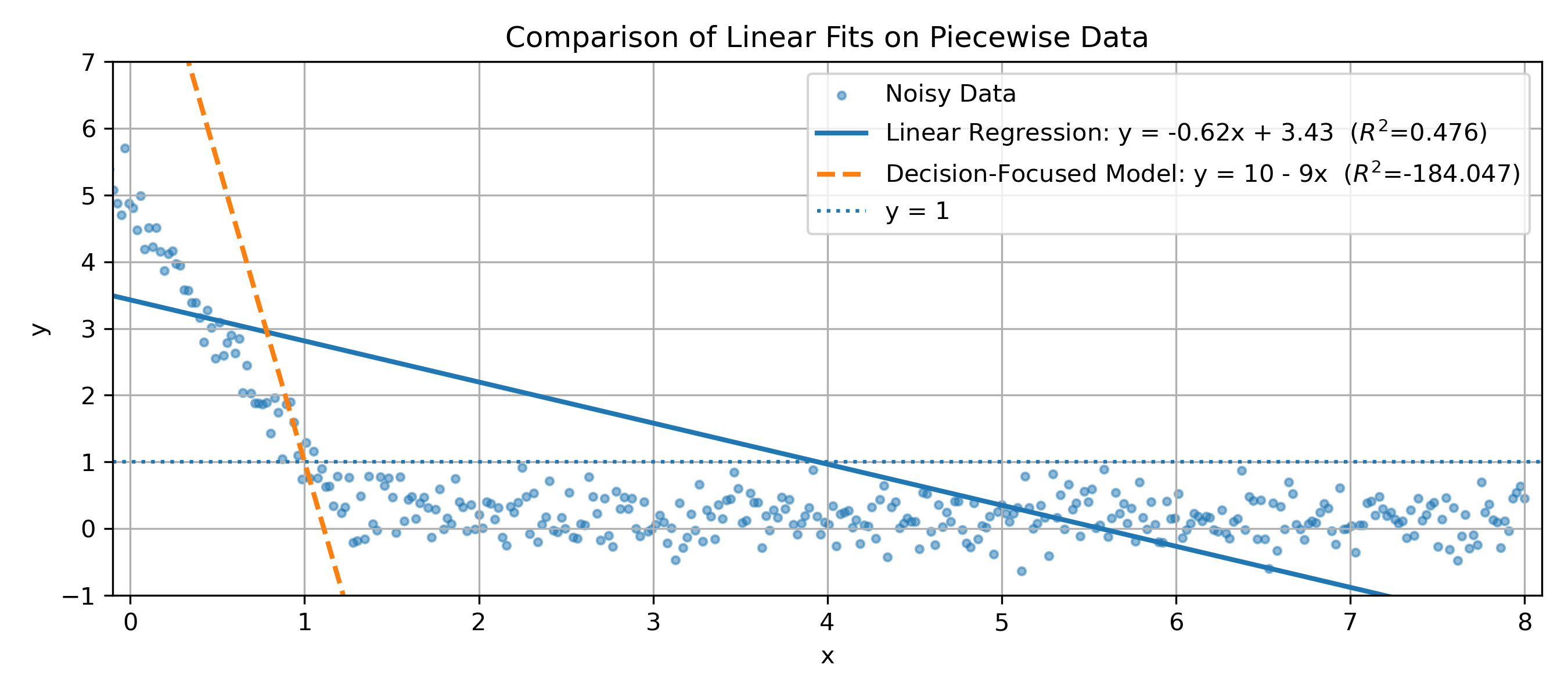}
			\end{minipage}
			
			\caption{A shortest-path toy example showing that better prediction fit need not imply better decisions.}
			\label{example:shortest}
		\end{figure}
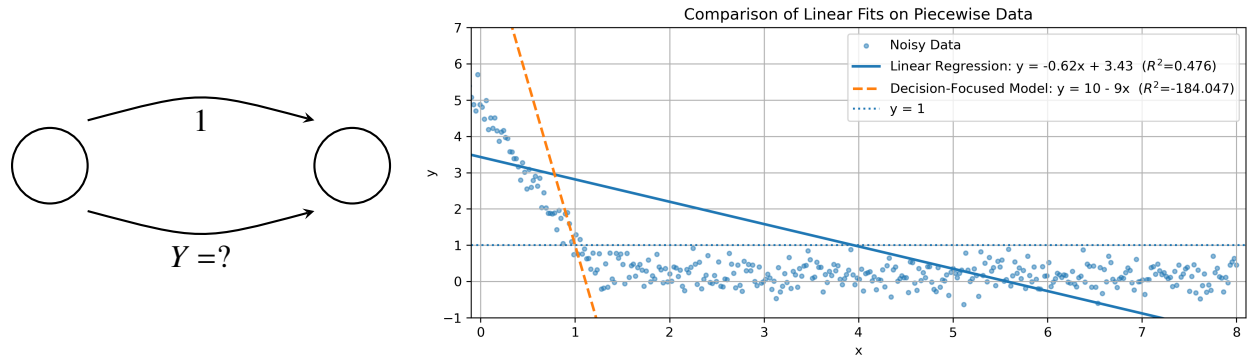
		
		Figure \ref{example:shortest} can be used to illustrate three cases in which DFL is necessary.
		
		\begin{itemize}
			\item \textbf{Smaller prediction error does not imply better decisions.}  
			In the right panel of Figure~\ref{example:shortest}, the blue least-squares fit achieves a much higher \(R^2\). However, the induced decision rule is suboptimal: it suggests choosing the upper route when \(x<4\), whereas the true threshold is around \(x<1\). In contrast, the dashed orange fit, despite having worse prediction accuracy, places the threshold correctly and therefore yields near-optimal decisions.
			
			\item \textbf{Uniform data collection can be inefficient.} When collecting data, we should focus on the region $x \in [0,2]$, rather than the entire space $[0,8]$. Data with $x > 4$ primarily improve prediction accuracy on $Y$, while the optimal decision is already clear. In contrast, observations of $Y$ near $x = 1$ are much more informative, as they directly reduce decision errors. This suggests that data collection should prioritize decision-relevant regions rather than uniformly covering the feature space.
			
			\item \textbf{Decision-relevant distributional differences require new metrics.}
			Quantifying distribution shifts or discrepancies is important for clustering, kernel regression, and decision robustness.
			Consider modifying the data distribution by shifting $y$ upward by 5 when $x<1$.
			Although this modification creates a large discrepancy under standard metrics such as KL divergence or Wasserstein distance, the optimal decision rule remains unchanged.
			Hence, geometric distances between distributions may fail to capture the differences that matter for decision-making.
		\end{itemize}
		
		These examples highlight a central theme of this tutorial: classical statistical learning tools, such as model training, data collection, and uncertainty quantification, should be reexamined from a decision-focused perspective. We study these questions in a broader class of decision-making problems, such as contextual linear programming.
		
		\subsection{Formulation and setup}\label{sec:formulation}
		
		The tractability of DFL depends on the structure of the downstream decision-making problem and how uncertainty enters the model. When the downstream problem has a nonlinear objective function or an uncertain feasible region, the statistical analysis of DFL becomes nontrivial. These challenges will be discussed in detail in Sections \ref{sec:hard_nonlinear} and \ref{sec:hard_rhs}. 
		
		We begin with a common setting studied in \citet{elmachtoub2022smart}. A more general setting and formulation will be discussed in Section \ref{sec:linear}.
		Suppose that we do not observe the true cost vector $c$ directly, but instead observe a feature vector $x \in \cX$ that can be used to predict $c$.
		A prediction model $h \in \cH$ outputs $\hat{c} = h(x)$, and the downstream decision is obtained by solving
		\begin{align}
			w^*(\hat c) \in \argmin_{w \in S} \ \hat c^\top w, \qquad S := \{w \in \bbR^d : Aw \le b\}. \label{eq:plugin-lp}
		\end{align}
		
		The matrix $A$ and the right-hand-side (RHS) vector $b$ are known parameters that define the constraints. The formulation in \eqref{eq:plugin-lp} has wide applications in OR, including \textit{shortest path, bipartite matching, sorting and ranking, top-$k$ selection, and max-flow/min-cut problems}.
		
		If the conditional distribution of $c$ given $x$ were known, the Bayes-optimal decision would be obtained by solving
		\begin{align}
			w^*(x) \in \argmin_{w\in S} \ \bE[c^\top w\mid x]
			= \argmin_{w\in S} \ \bE[c\mid x]^\top w. \label{eq:bayes-lp}
		\end{align}
		
		Equation~\eqref{eq:bayes-lp} reveals a key simplification of stochastic linear optimization: because the objective is linear in $c$, the entire conditional distribution enters only through the conditional mean $\bE[c \mid x]$. This implies that an accurate point estimate of the conditional mean suffices to yield optimal decisions that minimize the risk for a given feature $x$. This property later helps explain why stochastic linear programs are statistically more tractable than problems with nonlinear objectives or uncertain feasible regions.

		Given a hypothesis class of prediction models $\mathcal{H}$, statistical learning aims to identify the best model $h \in \mathcal{H}$ using data. Evaluating predictive performance of a prediction model $h$ requires comparing the predicted value $\hat{c} \gets h(x)$ with the true value $c$ under a loss function $\ell(\cdot,\cdot): \mathbb{R}^{d_1} \times \mathbb{R}^{d_1} \to \mathbb{R}$. The most common prediction-focused loss is the squared error,
		$\ell_2^2(h(x),c) = \|h(x)-c\|_2^2$.
		This choice is natural when the goal is accurate estimation of $\bE[c \mid x]$ itself.
		
		In predict-then-optimize settings, however, the final performance is determined not by $h(x)$ directly but by the decision induced by $h(x)$. To capture this, \citet{elmachtoub2022smart} introduced a decision-focused loss. Given a sample $(x,c)$, the per-instance decision loss is defined as
		\begin{align}
			\ell_{\mathrm{decision}}(h(x),c)
			:= c^\top w^*(h(x)) - c^\top w^*(c). \label{eq:spo-loss}
		\end{align}
		This loss, $\ell_{\mathrm{decision}}(h(x),c)$, is referred to as the \textit{Smart Predict-then-Optimize} (SPO) loss. In \eqref{eq:spo-loss}, the first term represents the realized cost of the decision $w^*(h(x))$ under the true cost vector $c$, while the second term corresponds to the minimum cost achievable if the true cost vector $c$ were known in advance.

		Suppose that $(x,c)$ are sampled from a fixed but unknown distribution $\cD$. The corresponding population risk under a loss function $\ell$ is defined as
		\begin{align}\label{equ:risk_expectation}
			R_\ell(h) := \bE_{(x,c)\sim\cD}[\ell(h(x),c)]. 
		\end{align}
		
		In DFL, we evaluate prediction models using the decision loss in $\eqref{eq:spo-loss}$. For simplicity, we denote the corresponding risk $R_{\ell_{\mathrm{decision}}}(h)$ in $\eqref{equ:risk_expectation}$ by $R(h)$. The goal of DFL is therefore to select a prediction model $h \in \cH$ that minimizes $R(h)$.
		
		Given a training set $\{(x_i,c_i)\}_{i=1}^n$, a straightforward DFL approach is to directly minimize the empirical decision risk:
		\begin{align}
			\min_{h \in \cH} \ \frac{1}{n}\sum_{i=1}^n \ld(h(x_i),c_i).
			\label{equ:empirical}
		\end{align}
		
		When the true cost vector $c_i$ is fully observed, the second term in \eqref{eq:spo-loss}, $c_i^\top w^*(c_i)$, is independent of the prediction $h(x_i)$. Consequently, it can be omitted from the optimization problem in \eqref{equ:empirical}, which reduces to
		\begin{align}
			\min_{h \in \cH} \ \frac{1}{n}\sum_{i=1}^n c_i^\top w^*(h(x_i)).
			\label{equ:reduce}
		\end{align}
		
		Note that this reduction is valid only when $c_i$ is fully observed. If the true cost vector is only partially observed, lies within an uncertainty set, or is chosen adversarially as a function of the decision, then the prediction model $h$ may depend on the choice of $c_i$. In such cases, we must work with the original decision loss $\ld$.

		When minimizing \eqref{equ:empirical} or \eqref{equ:reduce}, the resulting optimization problem remains computationally challenging. These challenges arise from the three steps shown in Figure~\ref{fig:flow}:
		
		\begin{itemize}
			\item Step (i): The prediction model (e.g., a neural network) may be nonconvex or nondifferentiable;
			\item Step (ii): The optimization solver $w^*$ may be computationally expensive, especially when the numbers of variables and constraints are large, and the mapping $h(x) \mapsto w^*(h(x))$ may be discontinuous and nondifferentiable;
			\item Step (iii): The true cost vector $c$ may not be fully observable.
		\end{itemize}
		Section~\ref{sec:review} briefly reviews these challenges, along with the main computational approaches in the literature for minimizing the decision loss in \eqref{equ:empirical} or its reduced form in \eqref{equ:reduce}.
		\begin{figure}[ht]
			\centering
			\begin{tikzpicture}[>=stealth, thick]
				\node (x) at (0,0) {$x$};
				\node (hx) at (3,0) {$h(x)$};
				\node (whx) at (6,0) {$w^*\big(h(x)\big)$};
				\node (final) at (10,0) {$c^\top w^*\big(h(x)\big)$};
				
				\draw[->] (x) -- node[midway, above] {$h$} (hx);
				\draw[->] (x) -- node[midway, below] {(i)} (hx);
				\draw[->] (hx) -- node[midway, above] {$w^*$} (whx);
				\draw[->] (hx) -- node[midway, below] {(ii)} (whx);
				\draw[->] (whx) -- node[midway, above] {$c$} (final);
				\draw[->] (whx) -- node[midway, below] {(iii)} (final);
			\end{tikzpicture}
			\caption{Three steps of evaluating the decision loss}
			\label{fig:flow}
		\end{figure}
		
		Because of these computational challenges, minimizing prediction error remains appealing in practice. The main focus of this tutorial is therefore to understand when the minimizer of the decision risk, $\min_{h\in \cH} R(h)$, differs from the minimizer of the prediction risk, $\min_{h\in \cH} R_{\ell_2^2}(h)$, and how this mismatch can be exploited to address statistical learning questions in DFL.

		\paragraph{Roadmap.}
		The remainder of the tutorial is organized as follows. Section~\ref{sec:review} briefly reviews computational methods for decision-focused learning and explains why directly optimizing decision loss can be difficult. Section~\ref{sec:linear} focuses on stochastic linear programming, first contrasting it with nonlinear objectives and stochastic feasible regions, and then summarizing the geometry and risk properties of the decision loss. Section~\ref{sec:datacollection} uses this geometry to revisit data collection and explains why decision-relevant samples are often concentrated near cone boundaries. Section~\ref{sec:distance} turns to uncertainty quantification and distributional comparison, showing why classical distances such as KL divergence and Wasserstein distance can be decision-blind and introducing a decision-focused alternative. The final sections discuss open statistical learning questions in DFL and conclude.
		
		\section{Computational Methods for DFL}\label{sec:review}
		
		In this section, we briefly review computational methods for DFL. DFL can be viewed as an approach to solving contextual stochastic optimization problems. In the literature, related terms include \textit{task-based learning}, \textit{decision-aware learning}, \textit{end-to-end learning}, \textit{operational statistics}, \textit{predict+optimize}, and \textit{smart predict-then-optimize} (\citet{honguyen2022risk,donti2017task,chu2008solving,feng2025contextual,bertsimas2020predictive,elmachtoub2022smart}). See the forthcoming textbook by \citet{gupta2026endtoendreader} for a more detailed discussion.
		
		The computational literature on DFL has grown rapidly. We refer readers to the recent surveys by \citet{mandi2024survey} and \citet{sadana2025survey} for comprehensive overviews. Here, we highlight several representative approaches to illustrate that computational challenges remain central to DFL, before turning to the statistical analysis in Sections~\ref{sec:linear}, \ref{sec:datacollection}, and \ref{sec:distance}.

		\subsection{Decision loss as the training objective: benefits and limitations}

		In DFL, the most straightforward approach is to minimize the decision loss in \eqref{equ:empirical} or \eqref{equ:reduce}, or a generalized version of these objectives in contextual stochastic optimization.
		This empirical risk is attractive because it naturally integrates the decision loss into the training process.  In particular, this approach can outperform prediction-focused training when the hypothesis class is misspecified, meaning that it does not contain the true underlying model that generates the data (e.g., using linear regression to fit a nonlinear trend).
		
		However, this integrated training approach has limitations from both computational and statistical perspectives. Computationally, as shown in Figure~\ref{fig:flow}, all three steps can be intractable because of nonconvexity or unobservability. By the chain rule, direct gradient descent may thereby be inefficient. Statistically, the second and third steps in Figure~\ref{fig:flow} involve mappings from the coefficient space to the decision space and objective values, which filter out much of the information contained in the predictions or observed data; see Section~\ref{sec:property} for details. Although this loss of information helps focus learning on decision quality in the misspecified case, it can make the learning process less efficient than prediction-focused approaches when the hypothesis class is well specified, that is, when the true underlying model lies within the hypothesis class. This slower statistical convergence rate, reflected in a larger out-of-sample risk bound, is studied in papers such as \citet{hu2022fast,elmachtoub2023stochasticdom,lan2025bias,elmachtoub2025dissecting,hu2025contextual}. 
		
		In light of these challenges, Section~\ref{sec:review_next} reviews several computational approaches to DFL.

		\subsection{Reviews of other computational methods}\label{sec:review_next}
		
		We review computational methods for DFL from the following four perspectives that are related to OR problems, and again refer readers to \citet{mandi2024survey} for a more detailed review.
		
		\paragraph{Surrogate loss approaches.} When the integrated loss is differentiable, for example, when the downstream decision-making problem is unconstrained or smoothly parameterized or regularized, gradient-based methods can be applied; see, for example, \citet{amos2017optnet,agrawal2019differentiable}. When the decision loss is nondifferentiable, for instance, because the optimal decisions $w^*$ may jump between extreme points in linear programming, researchers often use computationally tractable surrogate losses during training. Representative examples include the SPO+ loss of \citet{elmachtoub2022smart}, the risk-calibrated losses studied by \citet{honguyen2022risk}, the perturbation-gradient (PG) losses of \citet{huang2024directional}, the LAVA loss of \citet{berden2025solver}, the WISE loss of \citet{wan2026solverfree}, and the PEAR loss of \citet{lee2026decisionfocused}. In practice, the \textit{PyEPO} library developed by \citet{tang2022pyepo} provides a range of common surrogate losses and computational approaches that are readily implemented on a variety of datasets. A recent work, \citet{schneider2026soft} considers a radial projection method to address the zero gradient issue in neural network training.

		Instead of predicting every primitive uncertain quantity, the model can target a statistic or representation that is sufficient, or nearly sufficient, for decision-making. In stochastic linear optimization, the conditional mean $\bE[c \mid x]$ is already a decision-relevant target. In more structured applications, the relevant object may be a threshold, a quantile, or another low-dimensional summary. This viewpoint is especially useful when the raw parameter vector is high-dimensional, but the optimal policy depends only on a small subset of directions. Decision-relevant quantities can be grouped into the following three categories.
		
		\paragraph{Mapping $x\mapsto w^*$: End-to-end prediction of optimal decisions.}
		A natural simplification is to bypass the optimization layer at deployment and predict the action $w$ directly. This can be done through imitation learning, policy learning, or structured prediction of feasible decisions. For instance, \citet{wilder2019end,qi2023practical,liu2026inventory} train neural networks to output decisions, such as graph optimization solutions or replenishment decisions, directly from contextual information and historical data. Despite strong empirical performance, the theoretical guarantees depend heavily on the training loss functions and their alignment with the decision loss.
		
		\paragraph{Mapping $x\mapsto \ell(\cdot,\cdot)$: prediction of decision loss functions.}
		
		Rather than predicting the uncertain coefficient vector or the optimal decision directly, another line of work learns a surrogate loss that approximates the downstream decision regret and then trains the prediction model using this learned objective. Along this direction, \citet{wang2020compact} learn a compact, low-dimensional surrogate optimization layer, while \citet{shah2022decision} propose locally optimized decision losses (LODLs), which fit instance-specific convex losses using decision-loss evaluations from a black-box optimization oracle. Extending this idea, \citet{shah2024leaving} introduce efficient global losses (EGLs), which learn a feature-dependent map from contexts to loss parameters, allowing loss information to be shared across instances.
		
		\paragraph{Mapping $c\mapsto w^*(c)$: Prediction of optimization oracles.}
		Instead of solving the exact downstream problem at every training step, the learner can train a fast neural solver or meta-optimizer to emulate the mapping from problem parameters to near-optimal feasible decisions. Recent work by \citet{kotary2023predict,cristian2025metaopt} illustrates this approach. The promise is substantial computational savings, especially when the same optimization family is solved repeatedly during training. The main caveat is that approximation error in the learned oracle feeds back into the training objective, so theoretical guarantees must control both optimization error and statistical error.
		
		Despite the wide range of computational methods developed for diverse decision-making problems and forms of uncertainty, statistical guarantees for DFL remain available only for a much smaller class of problems. These challenges and results are discussed in Section~\ref{sec:linear}.
		
		\section{Decision-Focused Learning for Stochastic Linear Programming}\label{sec:linear}
		
		This section explains why stochastic linear programs with uncertain objective coefficients have become the canonical setting for statistical analyses of DFL. Let's go back to the linear formulation introduced in Section \ref{sec:formulation}, and discuss why statistical analysis becomes increasingly difficult when we have nonlinear objective functions or uncertain constraints. The following two sections illustrate why predicting the conditional mean of uncertainty does not minimize the decision risk $R(h)$ when the objective is nonlinear or when the uncertainty is in the feasible region. Since the classical newsvendor problem admits two equivalent formulations, one with a piecewise-linear objective and one with an uncertain right-hand-side, we use it repeatedly in Sections~\ref{sec:hard_nonlinear} and~\ref{sec:hard_rhs} to illustrate these challenges.

		\subsection{Statistical challenge for nonlinear objective}\label{sec:hard_nonlinear}
		
		Recall that for a linear objective, the decision is obtained by \eqref{eq:bayes-lp}, i.e., $
		w^*(x) \in \argmin_{w\in S} \ \bE[c^\top w\mid x]
		= \argmin_{w\in S} \ \bE[c\mid x]^\top w$.
		Suppose we generalize the objective function from $c^\top w$ to nonlinear function $\phi(\cdot,\cdot):\bbR^{d_2}\times \bbR^{d_1}\mapsto \bbR$, then the optimal decision should be obtained by
		\begin{align}
			w^*(x) \in \argmin_{w\in S} \bE[\phi(w,c)\mid x]. \label{eq:nonlinear-bayes}
		\end{align}
		This nonlinear objective can be motivated by the newsvendor cost, e.g., 
		\begin{align}\label{equ:newsvendor_cost}
			\phi(w,c) = \sum_{l=1}^{d} o_l ( w_l - c_l)^+ + b_l (c_l - w_l)^+,
		\end{align}
		where $o_l$ and $b_l$ denote the unit overstock and stockout costs, respectively, which may vary across products $l$.
		Another example of a nonlinear $\phi$ is the mean-variance portfolio optimization.
		The mean-variance portfolio optimization problem aims to maximize expected return while controlling risk. Suppose there are $d$ possible assets to invest in, and let $c\in\mathbb{R}^d$ denote the random return vector of the assets. The portfolio $w\in\mathbb{R}^d$ represents the percentage of total capital invested in each asset. The mean-variance objective can be written as
		\begin{align*}
			\min_{w\in S,\, w_0\in\mathbb{R}}
			\quad
			\mathbb{E}\left[
			-\rho c^\top w
			+
			\frac{\delta}{2}\left(c^\top w-w_0\right)^2
			\right],
		\end{align*}
		where $w_0$ is an auxiliary decision variable, $\rho>0$ and $\delta>0$ balance expected return and risk, and $S$ denotes the feasible set of portfolios.

		For nonlinear objectives, since $\bE[\phi(w,c)] \neq \phi(w,\bE[c])$, predicting the conditional mean generally does not lead to the optimal decision. Consequently, a statistically consistent approach is to estimate the full conditional distribution of $c$ and minimize the expected cost $\bE[\phi(w,c)]$ directly. This estimation task can be simplified when the distribution is parameterized by some parameter $\theta \in \Theta$, that is, when the data pair $(x,c)$ is drawn from some distribution $\cD_\theta$. In this case, the integrated estimation-and-optimization approach studied by \citet{elmachtoub2023stochasticdom,qi2025integrated} can be written as
		\begin{align*}
			\min_{\theta \in \Theta} \frac{1}{n} \sum_{i=1}^n \phi(w_\theta,c_i),
		\end{align*}
		where the optimization oracle $w_\theta$ is defined by
		\begin{align*}
			w_\theta \in \arg\min_{w \in S} \left\{ v(w,\theta) := \mathbb{E}_\theta[\phi(w,c)] \right\}.
		\end{align*}
		
		Although the integrated approach above is statistically consistent, in the sense that it converges to the best in-class parameter $\theta$, estimating the full distribution and minimizing $\mathbb{E}_\theta[\phi(w,c)]$ can be computationally intractable. To address this issue, we note that in some special problems, a point prediction of the uncertainty is sufficient for optimal decision-making. A classical example is the newsvendor problem, in which the optimal order quantity is given by the conditional quantile; see Example \ref{example:newsvendor} for details. Such point predictions can significantly simplify the learning stage; see \citet{ban2019big}, for example.

		When a decision-corrected point forecast exists, \citet{homem2024forecasting} study how to generate such forecasts. In general, however, one cannot expect a decision-corrected point prediction to exist for nonlinear objectives. For example, if the uncertainty is represented by a $d_1$-dimensional vector, whereas the decision lies in a $d_2$-dimensional space with $d_2 > d_1$, then a point prediction in $\mathbb{R}^{d_1}$ will generally not be sufficient to determine the optimal decision in $\mathbb{R}^{d_2}$.
		
		A recent study by \citet{er2025bias} is the first to consider necessary and sufficient conditions for the existence of a decision-corrected point estimate. They study this question in two related settings: a two-stage multi-item newsvendor problem and capacity design for a multi-period service system. \citet{er2025bias} show that the traditional fluid approximation, that is, using time-varying demand as the time-varying Poisson arrival rate, is decision-biased, and they propose conditions for checking the existence of a decision-corrected arrival rate.

		\subsection{Statistical challenge for stochastic feasible region}\label{sec:hard_rhs}
		
		Recall that in linear programming in \eqref{eq:plugin-lp}, the feasible region is formed by constraints $Aw\le b$. When matrix $A$ or RHS $b$ is random, in most OR problems, we cannot simply use $\bE[A]w\le \bE[b]$, as we need to guarantee the feasibility in each possible scenario or use a chance-constrained version. To illustrate this point, we return to the newsvendor problem in \eqref{equ:newsvendor_cost}.
		
		\begin{example}[Alternative formulation of the newsvendor problem]\label{example:newsvendor}
			Consider a simple system with a single customer class, a single resource pool, and a single period. Let $c$ be the ordering cost, and $p$ be the unit lost sales cost. To avoid triviality, suppose $c<p$; otherwise, purchasing no item would be optimal. Let $w$ be the order quantity, and then the traditional newsvendor problem can be written as the following two-stage problem: The first stage minimizes the ordering cost and the expected lost sales cost:
			\[
			\min_{w}: c w + \bE[\pi^*(w,D)].
			\]
			Given random demand $D$, the second-stage problem of minimizing the lost sales (or maximizing the satisfied demand) is
			\[
			\pi^*(w,D):=\;\min_{x\ge0}\; p\,(D-x)\quad\text{s.t.}\quad x\le D,~ x\le w,
			\]
			where $x$ denotes the number of sold items.
			The optimal recourse is $x^*(w,D)=\min(D,w)$, hence $\pi^*(w,D)=p(D-w)^+$. The first-stage problem is
			\[
			\min_{w\ge0}\;\; c\,w\;+\;\bE\!\left[p\,(D-w)^+\right].
			\]
			Denoting $F$ as the cumulative distribution function (CDF) of the demand $D$, for continuous $F$, the objective is convex and the first-order optimality gives $w^*=
			F^{-1}\!\big(\tfrac{p-c}{p}\big)$.
			\hfill\Halmos
		\end{example}

		Example~\ref{example:newsvendor} shows that an uncertain feasible region usually requires a two-stage formulation, rather than replacing both sides of the constraint by their expectations. From the newsvendor example above, we observe that when the right-hand-side is random, the conditional mean $\bE[D]$ is generally not the optimal point prediction. Instead, a certain conditional quantile serves as the decision-corrected point prediction.
		
		This distinction in DFL between uncertainty in the objective vector $c$ and uncertainty in the right-hand-side vector $b$ may appear counterintuitive from the perspective of strong duality. When the right-hand-side $\tilde b$ is stochastic, dualizing the primal recourse problem seems to move the uncertainty from the constraints into the objective of the dual problem:
		\[
		\begin{minipage}[c]{0.45\textwidth}
			\centering
			\[
			\raisebox{0.9em}{$\text{(Primal)}$}
			\qquad
			\begin{aligned}
				\min_{w} \quad & c^\top w \\
				\text{s.t.} \quad & Aw \le \tilde{b}
			\end{aligned}
			\]
		\end{minipage}
		\hfill
		\begin{minipage}[c]{0.45\textwidth}
			\centering
			\[
			\raisebox{1.2em}{$\text{(Dual)}$}
			\qquad
			\begin{aligned}
				\max_{y} \quad & -\tilde{b}^\top y \\
				\text{s.t.} \quad & A^\top y = -c, \\
				& y \ge 0
			\end{aligned}
			\]
		\end{minipage}
		\]
		
		Since the dual problem appears to fit the canonical DFL form in \eqref{eq:bayes-lp}, this distinction may seem puzzling at first glance. In fact, this misunderstanding is referred to as the \textit{``pitfall of stochastic right-hand-sides''} in \citet{er2025bias}, and is illustrated below.
		
		In particular, consider the following general two-stage optimization problem. We use $w^{\text{1st}}$ to denote the first-stage decision, and $w^{\text{2nd}}$ to denote the second-stage decision made after the uncertainty is observed. The first-stage cost is given by a deterministic convex function $f(w^{\text{1st}})$, while $g(w^{\text{1st}}, w^{\text{2nd}})$ denotes the second-stage cost, which depends on the realized uncertainty:
		\begin{align*}
			\min_{w^{\text{1st}}} \left\{ f(w^{\text{1st}}) + \mathbb{E}\left[\min_{w^{\text{2nd}}} g(w^{\text{1st}}, w^{\text{2nd}})\right]\right\}.
		\end{align*}

		Let \(A_1\in\mathbb{R}^{m\times n_1}\) and \(A_2\in\mathbb{R}^{m\times n_2}\) be known constraint matrices, and let
		\(w^{\mathrm{1st}}\in\mathbb{R}^{n_1}\) and \(w^{\mathrm{2nd}}\in\mathbb{R}^{n_2}\) denote the first- and second-stage decisions, respectively. We compare the following four stochastic two-stage problems:
		\begin{align}
			\min_{w^{\mathrm{1st}}}\;&
			f(w^{\mathrm{1st}})
			+
			\mathbb{E}_{c \sim \mathcal{D}_c}
			\Big[
			\min_{w^{\mathrm{2nd}}}
			c^\top w^{\mathrm{1st}}
			\;\;\mathrm{s.t.}\;\;
			A_1 w^{\mathrm{1st}} + A_2 w^{\mathrm{2nd}} \le b
			\Big],
			\tag{SP1}\label{append:p1}
			\\[6pt]
			\min_{w^{\mathrm{1st}}}\;&
			f(w^{\mathrm{1st}})
			+
			\mathbb{E}_{b \sim \mathcal{D}_b}
			\Big[
			\min_{w^{\mathrm{2nd}}}
			c^\top w^{\mathrm{1st}}
			\;\;\mathrm{s.t.}\;\;
			A_1 w^{\mathrm{1st}} + A_2 w^{\mathrm{2nd}} \le b
			\Big],
			\tag{SP2}\label{append:p2}
			\\[6pt]
			\min_{w^{\mathrm{1st}}}\;&
			f(w^{\mathrm{1st}})
			+
			\mathbb{E}_{c \sim \mathcal{D}_c}
			\Big[
			\min_{w^{\mathrm{2nd}}}
			c^\top w^{\mathrm{2nd}}
			\;\;\mathrm{s.t.}\;\;
			A_1 w^{\mathrm{1st}} + A_2 w^{\mathrm{2nd}} \le b
			\Big],
			\tag{SP3}\label{append:p3}
			\\[6pt]
			\min_{w^{\mathrm{1st}}}\;&
			f(w^{\mathrm{1st}})
			+
			\mathbb{E}_{b \sim \mathcal{D}_b}
			\Big[
			\min_{w^{\mathrm{2nd}}}
			c^\top w^{\mathrm{2nd}}
			\;\;\mathrm{s.t.}\;\;
			A_1 w^{\mathrm{1st}} + A_2 w^{\mathrm{2nd}} \le b
			\Big].
			\tag{SP4}\label{append:p4}
		\end{align}
		Here, \(c\) has the appropriate dimension depending on whether it multiplies \(w^{\mathrm{1st}}\) or \(w^{\mathrm{2nd}}\). In \eqref{append:p1} and \eqref{append:p3}, the uncertainty lies in the objective coefficient vector \(c\). In contrast, in \eqref{append:p2} and \eqref{append:p4}, the uncertainty lies in the right-hand-side vector \(b\), and hence affects the feasible region of the second-stage problem. 

		\paragraph{Why is \eqref{append:p1} easier?}
		In \eqref{append:p1}, the random coefficient \(c\) multiplies the first-stage decision \(w^{\mathrm{1st}}\), while the feasible region of the second-stage problem is deterministic. Thus, whenever the recourse problem is feasible, the inner objective is simply linear in \(c\). Consequently,
		\[
		\mathbb{E}_{c\sim\mathcal{D}_c}
		\Big[
		\min_{w^{\mathrm{2nd}}}
		c^\top w^{\mathrm{1st}}
		\;\;\mathrm{s.t.}\;\;
		A_1 w^{\mathrm{1st}} + A_2 w^{\mathrm{2nd}} \le b
		\Big]
		=
		\min_{w^{\mathrm{2nd}}}
		\mathbb{E}[c]^\top w^{\mathrm{1st}}
		\;\;\mathrm{s.t.}\;\;
		A_1 w^{\mathrm{1st}} + A_2 w^{\mathrm{2nd}} \le b .
		\]
		Therefore, a point prediction \(\hat c \approx \mathbb{E}[c]\) is sufficient, and \eqref{append:p1} reduces to a deterministic optimization problem.
		
		\paragraph{Why is \eqref{append:p2} harder?}
		Problem \eqref{append:p2} differs from \eqref{append:p1} because the uncertainty appears in the RHS of the second-stage constraints. Even though the objective is still linear in \(w^{\mathrm{1st}}\), the feasibility of \(w^{\mathrm{1st}}\) now depends on the realized value of \(b\). Replacing \(b\) by its mean can therefore be misleading: a decision that is feasible for \(\mathbb{E}[b]\) may be infeasible for many realizations of \(b\). Thus, the main difficulty in \eqref{append:p2} is not the evaluation of a nonlinear recourse cost, but the need to enforce feasibility across scenarios, either almost surely, robustly, or with high probability through a chance constraint.
		
		\paragraph{Why is \eqref{append:p3} harder?}
		Problem \eqref{append:p3} differs from \eqref{append:p1} because the random coefficient \(c\) multiplies the adaptive second-stage decision \(w^{\mathrm{2nd}}\). For a fixed first-stage decision \(w^{\mathrm{1st}}\), expectation and minimization generally do not commute:
		\[
		\mathbb{E}_{c\sim\mathcal{D}_c}
		\left[
		\min_{w^{\mathrm{2nd}}}
		c^\top w^{\mathrm{2nd}}
		\;\;\mathrm{s.t.}\;\;
		A_1 w^{\mathrm{1st}} + A_2 w^{\mathrm{2nd}} \le b
		\right]
		\le
		\min_{w^{\mathrm{2nd}}}
		\mathbb{E}[c]^\top w^{\mathrm{2nd}}
		\;\;\mathrm{s.t.}\;\;
		A_1 w^{\mathrm{1st}} + A_2 w^{\mathrm{2nd}} \le b .
		\]
		Hence, predicting only the conditional mean of \(c\) is generally insufficient. The full distribution of \(c\) matters because different realizations of \(c\) may induce different optimal second-stage decisions. Nevertheless, the second-stage value in \eqref{append:p3} is convex in \(w^{\mathrm{1st}}\). Therefore, when \(f\) is convex and the recourse problem is well behaved, \eqref{append:p3} remains a convex stochastic optimization problem, even though it cannot be reduced to a deterministic problem by a mean plug-in.
		
		\paragraph{Why is \eqref{append:p4} harder?}
		Problem \eqref{append:p4} combines an adaptive second-stage objective with uncertainty in the right-hand-side. Let
		\[
		v(b;w^{\mathrm{1st}})
		:=
		\min_{w^{\mathrm{2nd}}}
		\left\{
		c^\top w^{\mathrm{2nd}}
		:
		A_1 w^{\mathrm{1st}} + A_2 w^{\mathrm{2nd}} \le b
		\right\}
		\]
		denote the second-stage value function. For fixed \(w^{\mathrm{1st}}\), the map \(b\mapsto v(b;w^{\mathrm{1st}})\) is generally piecewise linear and convex. Therefore, Jensen's inequality gives
		\[
		\mathbb{E}_{b\sim\mathcal{D}_b}
		\big[
		v(b;w^{\mathrm{1st}})
		\big]
		\ge
		v\big(\mathbb{E}[b];w^{\mathrm{1st}}\big).
		\]
		As a result, minimizing
		$f(w^{\mathrm{1st}})
		+
		v\big(\mathbb{E}[b];w^{\mathrm{1st}}\big)$
		can lead to a suboptimal first-stage decision. Moreover, the active constraints and the associated dual variables depend on the realized scenario \(b\).

		\newcolumntype{Y}{>{\raggedright\arraybackslash}X}
		
		\begin{table}[h]
			\centering
			\scriptsize
			\setlength{\tabcolsep}{2.6pt}
			\renewcommand{\arraystretch}{1.18}
			\begin{tabularx}{\linewidth}{@{}l l l c c c Y l@{}}
				\toprule
				Problem
				& Randomness
				& Entry point
				& \shortstack{Mean\\plug-in}
				& \shortstack{Potential scenario\\infeasibility}
				& \shortstack{Convex in\\\(w^{\mathrm{1st}}\)}
				& Example applications
				& Relation \\
				\midrule
				\eqref{append:p1}
				& \(c\)
				& \shortstack[l]{ First- or second-stage\\ objective}
				& \cmark
				& \xmark
				& \cmark
				& Canonical DFL; portfolio; routing; allocation
				& Baseline \\
				
				\eqref{append:p2}
				& \(b\)
				& Second-stage RHS
				& \xmark
				& \cmark
				& \cmark
				& Newsvendor; capacity planning; staffing; service levels
				& Special case of \eqref{append:p4} \\
				
				\eqref{append:p3}
				& \(c\)
				& Second-stage objective
				& \xmark
				& \xmark
				& \cmark
				& Multistage transportation; dispatch; assignment
				& Dual-related to \eqref{append:p4} recourse \\
				
				\eqref{append:p4}
				& \(b\)
				& Second-stage RHS
				& \xmark
				& \cmark
				& \cmark
				& Stochastic games; Markov chain steady states; network
				& General RHS-recourse case \\
				\bottomrule
			\end{tabularx}
			\caption{Comparison of four stochastic two-stage formulations.}
			\label{tab:sp-comparison}
		\end{table}

		The comparisons are summarized in Table~\ref{tab:sp-comparison}.
		Because predicting only the conditional mean of the uncertainty is generally insufficient when the feasible region is stochastic, a statistically reliable approach for handling uncertain feasible regions, or uncertain right-hand-sides, is to estimate the full conditional distribution in DFL. In DFL, because the stochastic feasible region forms a two-stage problem, its formulation can be connected to bilevel optimization, e.g., \citet{bucarey2024decision,qi2025integrated}.

		However, estimating the full distribution and solving the recourse problem potentially infinitely many times are computationally inefficient. This motivates the question of whether one can identify a point statistic that replaces the full distribution while still yielding the optimal first-stage decision. Although the single-item newsvendor problem enjoys this property, such decision-corrected point statistics do not extend in general to the multi-item newsvendor problem, as shown by \citet{er2025bias}. They provide counterexamples in which \emph{no} point estimate of the uncertain right-hand-side reproduces the optimal stochastic decision.
		
		The broader lesson is that uncertainty in the constraints introduces additional challenges beyond those in \eqref{append:p1}. Duality does not resolve this issue because the dual multipliers depend on the realized scenario, and evaluating the expected recourse value generally requires more than a single plug-in vector. In practice, if one allows a post-projection step to ensure feasibility, some computational methods can still achieve good empirical performance; see, for example, \citet{hu2023two}.
		
		\subsection{Motivation for focusing on stochastic (mixed integer) linear programs}
		
		Because nonlinear objectives and stochastic feasible regions are already difficult at the level of statistical representation, much of the recent theory focuses on the simpler downstream problem in \eqref{eq:bayes-lp}, with known feasible region $S$.
		This model retains the essential predict-then-optimize structure while avoiding the representation problem above: a Bayes-optimal point prediction exists and is simply $\bE[c\mid x]$.

		Another advantage of \eqref{eq:bayes-lp} is geometric. If $S$ is a bounded polyhedron with extreme points $w_{(1)},\dots,w_{(m)}$, then the cost space is partitioned into normal cones
		\[
		N_j := \{c\in\bbR^d : w_{(j)} \in \argmin_{w\in S} c^\top w\}, \qquad j=1,\dots,m.
		\]
		This partition is illustrated in Figure \ref{fig:partition}.
		\begin{figure}[H]
			\centering
			\includegraphics[width=0.55\linewidth]{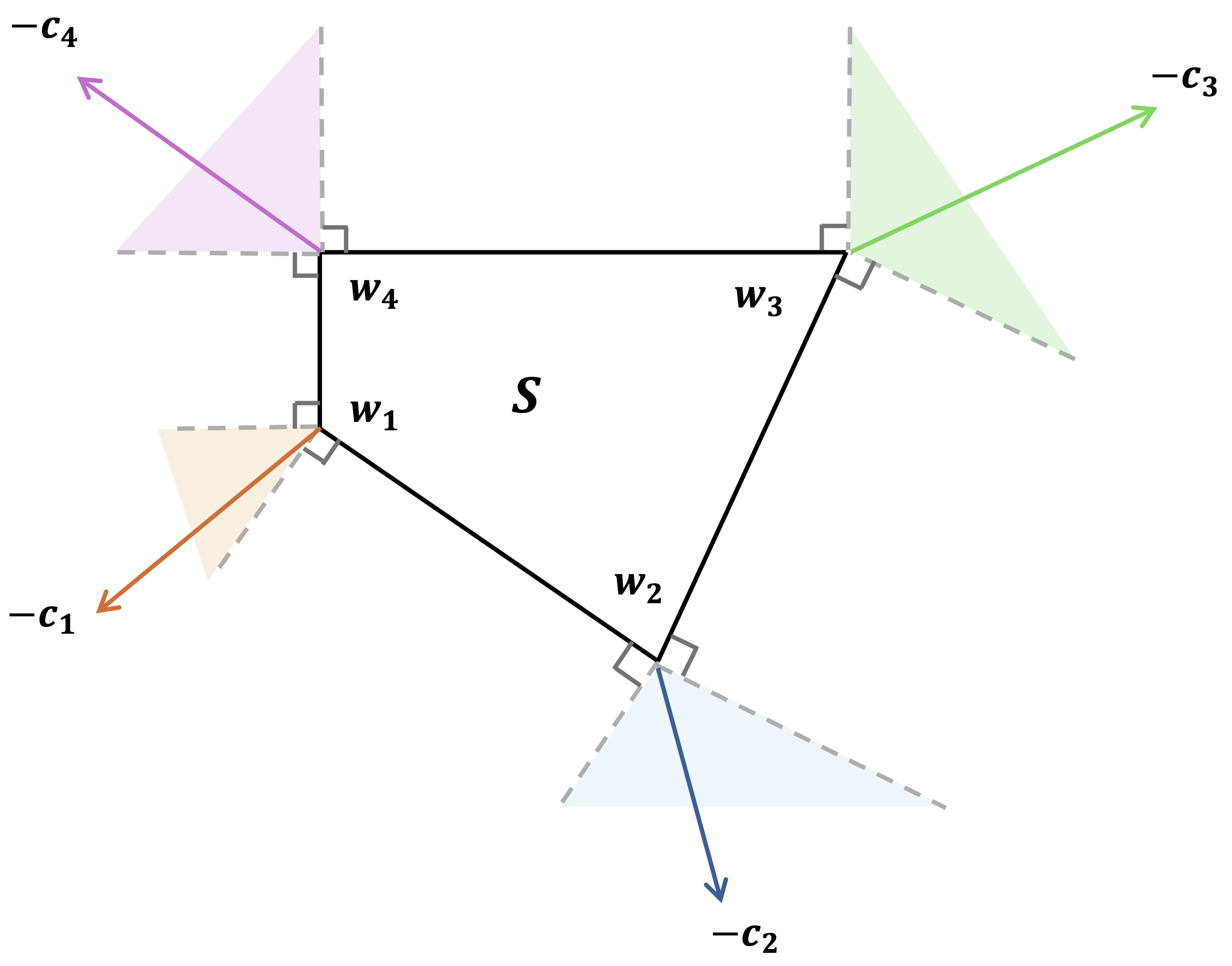}
			\includegraphics[width=0.42\linewidth]{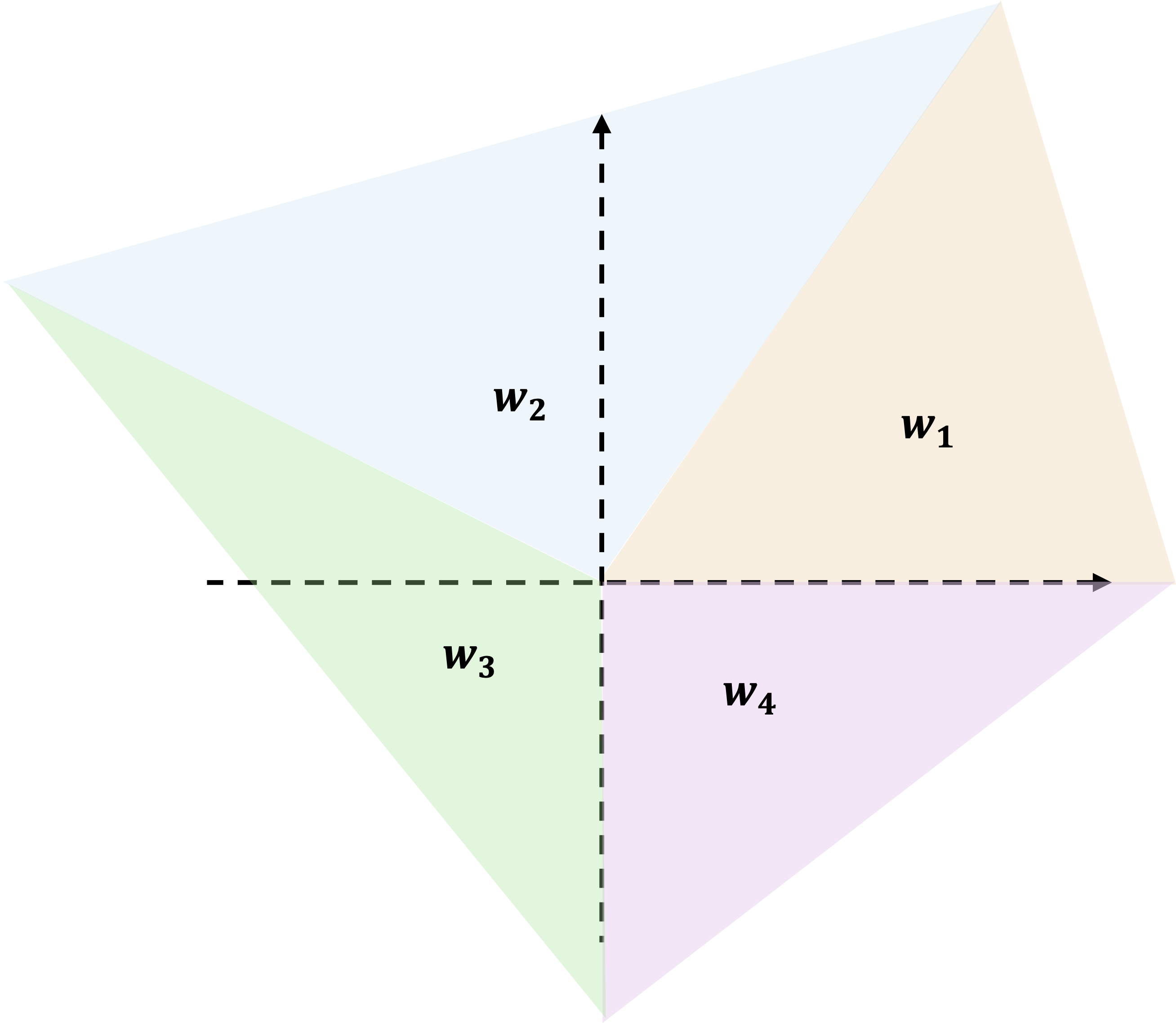}
			\caption{Geometric interpretation of $w^*(c)$ (left) and the partition of the cost space induced by optimal extreme points (right).}
			\label{fig:partition}
		\end{figure}

		In Figure~\ref{fig:partition}, the feasible region $S$ is represented by a polyhedron with four extreme points. The reverse directions of four example cost vectors are $-c_1, -c_2, -c_3$, and $-c_4$ in two dimensions. Minimizing the objective $c^\top w$ over $S$ intuitively selects the extreme point that lies farthest in the direction $-c$ within the feasible region. Thus, $w_1, w_2, w_3$, and $w_4$ are the optimal decisions under cost vectors $c_1, c_2, c_3$, and $c_4$, respectively. Accordingly, the entire two-dimensional cost space can be partitioned into four cones, each of which corresponds to one optimal extreme point.
		
		Because the optimal solutions can be generated from a finite set of
		extreme points, the same partitioning idea continues to hold when this set is extended to feasible integer solutions. Therefore, the analysis in the remainder of this tutorial also applies to (mixed-)integer linear programming. Such integer formulations arise in a broader class of OR problems, including knapsack, the traveling salesperson problem, combinatorial portfolio optimization, diverse bipartite matching, and energy-cost-aware scheduling; see \citet{mandi2024survey} for additional examples and discussion.
		
		Predictions that remain in the same cone induce the same decision, even if they differ substantially in Euclidean distance. Decision errors occur only when the predicted cost vector crosses from one cone to another. This cone geometry turns the downstream optimization problem into a weight-sensitive multiclass classification problem \citet{elmachtoub2022smart,liu2023active}, where each feature $x$ is classified into the extreme point $w^*(h(x))$.
		
		For these reasons, stochastic linear optimization has become a natural laboratory for studying the basic statistical questions of DFL, including calibration of surrogate losses, generalization bounds, sample complexity, active data collection, and decision-focused measures of distributional discrepancy (\citet{hu2022fast,elbalghiti2023generalization,liu2023active,wan2026directional}). The setting is simple enough to permit sharp theory, yet still rich enough to reveal precisely why classical learning tools can fail.
		
		In the next section, we illustrate several key properties of the decision loss function $\ld$ in stochastic linear programming.
		
		\subsection{Properties of the decision loss}\label{sec:property}
		
		Recall that, for linear-objective uncertainty, the per-instance decision loss is
		\[
		\ld(\hat c,c)
		:= c^\top w^*(\hat c)-c^\top w^*(c)
		\]

		The purpose of this section is to summarize several useful properties of the decision loss $\ld(\cdot,\cdot)$. These properties highlight its key differences from standard prediction losses, such as the squared loss, and provide useful tools for statistical analysis.
		Throughout, let $\operatorname{diam}(S):=\sup_{u,v\in S}\|u-v\|$ denote the size of the feasible region. We impose the following standard and mild assumptions.

		\begin{assumption}[Compact feasible set and fixed tie-breaking]\label{assumption:tie-break}
			Throughout this subsection, \(S\subseteq \mathbb{R}^d\) is nonempty, compact, and has finite extreme points. For each \(c\in\mathbb{R}^d\), \(w^*(c)\) denotes a deterministic selection from \(\argmin_{w\in S} c^\top w\), obtained via a fixed tie-breaking rule. Under this assumption, \(w^*(c)\) is well defined for every \(c\). 
		\end{assumption}

		\begin{theorem}[Properties of $\ld$ ]\label{thm:prop1} Under Assumption \ref{assumption:tie-break}, for any \(c,\hat c\in\mathbb{R}^d\), the following hold:
			
			\begin{enumerate}[label=(\thetheorem.\arabic*)]
				\item \textbf{Nonnegativity:}
				\[
				\ld(\hat c,c)\ge 0.
				\]
				\item \label{1.2}\textbf{Positive-scale invariance in the prediction:} for every \(\alpha>0\),
				\[
				\ld(\alpha \hat c,c)=\ld(\hat c,c).
				\]
				\item \label{1.3}\textbf{Positive homogeneity in the true cost:} for every \(\alpha\ge 0\),
				\[
				\ld(\hat c,\alpha c)=\alpha\,\ld(\hat c,c).
				\]
				\item \label{1.4}\textbf{Asymmetry:} there exists $c$ and $\hat c$,
				\[
				\ld(\hat c,c)\neq \ld(c,\hat c).
				\]
				\item \label{1.5}\textbf{Prediction-to-decision upper bound:} for any norm \(\|\cdot\|\) with dual norm \(\|\cdot\|_*\),
				\[
				0\le \ld(\hat c,c)
				\le \|\hat c-c\|_*\,\|w^*(\hat c)-w^*(c)\|
				\le \operatorname{diam}(S)\,\|\hat c-c\|_*.
				\]
				\item \label{1.6}\textbf{Dependence on true cost vector for a general convex feasible region. } The map $ c \longmapsto \ld(\hat c,c)$
				is finite, continuous, convex, and Lipschitz. 
				
				\item \label{1.7}\textbf{Dependence on true cost vector for a polyhedral feasible regions. } The map $ c \longmapsto \ld(\hat c,c)$ is piecewise linear, and more precisely
				\[
				\ld(\hat c,c)
				= c^\top w^*(\hat c)-\min_{w\in S} c^\top w
				= \max_{j=1,\dots,m} c^\top\big(w^*(\hat c)-w_{(j)}\big).
				\]
				\item \label{1.8}\textbf{Dependence on prediction. } The map \(\hat c\mapsto \ld(\hat c,c)\) is piecewise constant. In particular, its gradient with respect to \(\hat c\) is zero on each cone interior, while jumps may occur at cone boundaries. Thus, the map \(\hat c\mapsto \ld(\hat c,c)\) is generally nonconvex and discontinuous.
			\end{enumerate}
		\end{theorem}
		
		Note that Theorems \ref{1.2}, \ref{1.3} and \ref{1.4} illustrate different geometric properties from the squared loss. The scale-invariant property in \ref{1.2} is the key for designing data collection methods in DFL, shown in Section \ref{sec:datacollection}. Theorem \ref{1.5} shows that the decision loss can be well-controlled by either of the two terms, the prediction error $\|\hat c - c\|_*$ or the decision error $\|w^*(\hat c)-w^*(c)\|$. Theorems (1.1-1.6) can be generalized to the convex feasible region, not necessarily a polyhedron.

		Theorem \ref{1.8} implies that the empirical SPO risk can be flat on large regions and discontinuous at cone boundaries. Consequently, direct gradient-based optimization of the exact loss is difficult, even when the prediction model itself is smooth.
		
		Theorems \ref{1.5} and \ref{1.8} imply that zero prediction error implies zero decision loss, but the converse fails. Any prediction \(\hat c\) that lies in the same normal cone as \(c\) induces the same optimizer and therefore incurs zero decision loss. Thus, many predictions with nonzero Euclidean error are decision-equivalent.
		
		The proof of most properties in Theorem \ref{thm:prop1} can be found in \citet{elmachtoub2022smart,liu2023active,wan2026directional,elbalghiti2023generalization}. Here, we only provide the proof of Theorem \ref{1.5} for illustration.

		\emph{\bfseries Proof of Theorem \ref{1.5}} The proof is by the decomposition of $\ld$ and the Cauchy–Schwarz inequality.
		\begin{align*}
			\ld(\hat c,c)
			&= c^\top\big(w^*(\hat c)-w^*(c)\big) \\
			&= (c-\hat c)^\top\big(w^*(\hat c)-w^*(c)\big)
			+ \hat c^\top\big(w^*(\hat c)-w^*(c)\big) \\
			&\le (c-\hat c)^\top\big(w^*(\hat c)-w^*(c)\big) \\
			&\le \|\hat c-c\|_*\,\|w^*(\hat c)-w^*(c)\|.
		\end{align*}
		\hfill\(\square\)

		\subsection{Properties of the expected decision loss}\label{sec:expected_property}
		
		In this section, we further summarize some properties of $\ld$ under randomness. We ignore the dependence on the contextual information $x$ for notational convenience. All the following properties also hold when conditional on feature $x$.
		Let
		\[
		z(c):=\min_{w\in S} c^\top w,
		\qquad
		R(\hat c):=\bE[\ld(\hat c,c)],
		\qquad
		\mu:=\bE[c].
		\]
		
		In the stochastic setting, given a prediction $\hat c$, we cannot simply use the mean $\mu$ to evaluate the decision risk. Indeed, the mean-evaluated risk, $\ld(\hat c,\mu)$, may be smaller than the true risk $R(\hat c)$, and therefore may underestimate the actual decision risk. Theorem \ref{thm:prop2} summarizes several properties of $\ld$ under randomness.
		
		\begin{theorem}[Properties of $\ld$ under uncertainty]\label{thm:prop2}
			Under Assumption \ref{assumption:tie-break}, suppose that $c$ follows a fixed distribution. For any \(\hat c\in\mathbb{R}^d\), the following hold:
			\begin{enumerate}[label=(\thetheorem.\arabic*)]
				
				\item \textbf{Nonnegativity of risk.}
				For every \(\hat c\in\mathbb{R}^d\), we have $R(\hat c)\ge 0$.
				
				\item \label{2.2}\textbf{Risk decomposition (excess risk identity).}
				$R(\hat c)-R(\mu)=\ld(\hat c,\mu)$.
				
				\item \label{2.3}\textbf{Plug-in lower bound (Jensen gap characterization).}
				$R(\mu)=z(\mu)-\bE[z(c)]\ge 0$. Equivalently, $\bE[\ld(\bE[c],c)]\ge0$.
				
				\item \label{2.4}\textbf{Lower bound via plug-in decision loss.}
				$R(\hat c)\ge \ld(\hat c,\mu)$.
				
				\item \label{2.5}\textbf{Flatness of expected loss (a.e., zero gradient).}
				\[
				\nabla_{\hat c}\,\bE[\ld(\hat c,c)]=0
				\quad \text{for a.e. } \hat c.
				\]
				
				\item \label{2.6}\textbf{Monotonicity along rays (directional monotonicity).}
				\[
				\ld(c+\Delta,c)\le \ld(c+2\Delta,c),
				\]
				and more generally $t\mapsto \ld(c+t\Delta,c)$ is nondecreasing on $[0,\infty)$.
				
			\end{enumerate}
		\end{theorem}
		
		Theorem \ref{2.2} shows that the excess risk is the decision loss between the prediction $\hat c$ and the true mean $\mu$. Theorem \ref{2.3} shows that the Bayesian risk can be nonzero. Theorem \ref{2.4} shows that if we just use the plug-in mean to evaluate the decision risk, this evaluation can be pessimistic in practice. Theorem \ref{2.5}  shows that the gradient is zero almost everywhere. Theorem \ref{2.6} shows that in general, when the prediction is further away from the true mean, the decision loss increases. 
		
		Most of the proof can be found in \citet{liu2023active,wan2026directional}. Here, we provide the proof of Theorems \ref{2.2}, \ref{2.3}, and \ref{2.6} for illustration.
		
		\emph{\bfseries Proof of Theorems \ref{2.2}, \ref{2.3}, and \ref{2.6}.}
		
		Since \(z(c)=c^\top w^*(c)\), we have
		\[
		R(\hat c)
		= \bE\!\left[c^\top w^*(\hat c)-z(c)\right]
		= \mu^\top w^*(\hat c)-\bE[z(c)].
		\]
		Therefore,
		\begin{align}
			R(\hat c)
			&= \bigl(\mu^\top w^*(\hat c)-z(\mu)\bigr)
			+ \bigl(z(\mu)-\bE[z(c)]\bigr) \notag\\
			&= \ld(\hat c,\mu) + \bigl(z(\mu)-\bE[z(c)]\bigr).
			\label{eq:risk-decomposition}
		\end{align}
		Because \(z(c)=\min_{w\in S} c^\top w\) is the pointwise minimum of linear functions, it is concave. Hence, by Jensen's inequality, we have $z(\mu)\ge \bE[z(c)]$.
		Equation~\eqref{eq:risk-decomposition} gives
		\[
		R(\hat c)=\ld(\hat c,\mu)+\bigl(z(\mu)-\bE[z(c)]\bigr)\ge0.
		\]
		Taking \(\hat c=\mu\) yields \ref{2.3}, and subtracting \(R(\mu)\) from \(R(\hat c)\) gives \eqref{2.2}. Because \(\ld(\hat c,\mu)\ge 0\), \(\mu\) is a minimizer of \(R\). 
		
		Next, to prove \ref{2.6}, fix \(0\le s<t\), and let
		\[
		w_s:=w^*(c+s\Delta),
		\qquad
		w_t:=w^*(c+t\Delta).
		\]
		By optimality,
		\begin{align}
			(c+s\Delta)^\top w_s &\le (c+s\Delta)^\top w_t, \label{eq:mono1}\\
			(c+t\Delta)^\top w_t &\le (c+t\Delta)^\top w_s. \label{eq:mono2}
		\end{align}
		Let
		\[
		A:=c^\top(w_t-w_s),
		\qquad
		B:=\Delta^\top(w_t-w_s).
		\]
		Then \eqref{eq:mono1} and \eqref{eq:mono2} become
		\[
		A+sB\ge 0,
		\qquad
		A+tB\le 0.
		\]
		Since \(t>s\), these inequalities imply \(B\le 0\), and therefore
		\[
		A\ge -sB\ge 0.
		\]
		Thus
		\[
		c^\top w_t\ge c^\top w_s,
		\]
		which is equivalent to
		\[
		\ld(c+s\Delta,c)\le \ld(c+t\Delta,c).
		\]
		Taking \(s=1\) and \(t=2\) gives the stated inequality. \hfill\(\square\)

		These insights from Theorems~\ref{thm:prop1} and~\ref{thm:prop2} serve as the foundation for addressing statistical learning questions in DFL. The next sections revisit these questions through data collection, uncertainty quantification, and distances between distributions.

		\section{Data Collection for Decision-Focused Learning}\label{sec:datacollection}
		
		Data collection is a central problem in statistical learning, where one studies how many samples, and which kinds of samples, are needed to ensure that prediction quality or decision quality exceeds a desired threshold. When samples are collected i.i.d., the stopping time for data collection follows directly from the sample complexity or generalization error bound. However, as shown in the toy example in Figure \ref{example:shortest}, collecting data i.i.d. is not efficient for DFL. A more efficient data collection method is to focus on the decision-relevant samples. This sample selection is an important question
		whenever labels are expensive. In the present setting, a ``label'' refers to a realization of the uncertain quantity that enters the downstream optimization problem, such as a cost vector or a demand vector. In many applications, observing this quantity requires running a costly experiment, collecting detailed operational data, or waiting for a stochastic outcome to be realized. This naturally leads to \textit{active-learning} and \textit{sequential-experimental-design} questions: which contexts $x$ should we label, and when should we stop collecting data?
		
		In traditional statistical learning, data collection usually prioritizes points with higher prediction uncertainty. A key message of DFL, however, is that informative points are not necessarily those with a large prediction uncertainty. Rather, they are the points whose uncertainty is most likely to change the downstream decision. As shown in Figure \ref{fig:partition}, the decision map $c\mapsto w^*(c)$ is constant inside each normal cone and changes only when the predicted cost vector crosses a cone boundary. This geometry motivates two complementary approaches to data collection. The first is the \emph{margin-based method} in \citet{liu2023active}: quantify how far the current prediction is from degeneracy and collect more data only when the current confidence region intersects a cone boundary; see Section~\ref{sec:margin}. The second is the \emph{direction-based method} in \citet{wan2026directional}: exploit the scale invariance of the decision loss and measure predictive disagreement through normalized directions rather than Euclidean distance; see Section \ref{sec:direction}.
		
		\subsection{Problem setup and motivation}
		
		Let $\mathcal{X}$ denote the feature space, or a pool of unlabeled contexts. In the setting of experimental design, these unlabeled features are often called design points. For each $x\in\mathcal{X}$, we may pay a cost to observe a realization of the uncertain quantity. The goal is to learn a predictor $h(x)$ while spending as few labels as possible, subject to achieving low downstream decision loss. This formulation covers pool-based active learning, sequential experimental design, adaptive simulation, and online data-acquisition problems.
		
		Let $h_t(x)$ be the current estimate after $t$ labels. Suppose further that the learner maintains a confidence ball, around $h_t(x)$, for example
		\[
		\mathcal{B}_t(x)
		:= \{u\in\bbR^d:\ \|u-h_t(x)\|\le r_t(x)\},
		\]
		where $r_t(x)$ is a radius summarizing the remaining statistical uncertainty. These confidence balls are illustrated by the green circles in Figure~\ref{fig:margin_based}. Intuitively, each confidence ball is constructed so that the true conditional mean $\bE[c\mid x]$ lies within the ball with high probability. In practice, $r_t(x)$ can be derived from concentration inequalities or generalization error bounds, and it typically decays at the rate \(O(1/\sqrt{n})\) when \(n\) samples are collected. It can also be obtained through empirical methods or conformal prediction.
		
		A prediction-focused strategy would collect labels with the largest radius $r_t(x)$. In DFL, however, a large radius is important only if it can change the optimal decision induced by the current prediction. This is why uncertainty must be measured relative to the geometry of the feasible region.
		
		\subsection{Margin-based data collection}\label{sec:margin}
		
		The margin-based view quantifies how far a predicted cost vector is from a \emph{decision boundary}. Recall that when the feasible region $S$ is a bounded polyhedron with extreme points $w_{(1)},\dots,w_{(m)}$, the cost space is partitioned into normal cones
		\[
		N_j:=\{u\in\bbR^d: w_{(j)}\in\argmin_{w\in S} u^\top w\},
		\qquad j=1,\dots,m.
		\]
		Inside the interior of each cone, the oracle $w^*(u)$ is constant. Decision uncertainty, therefore, concentrates near the union of cone boundaries.
		
		\begin{definition}[Distance to degeneracy, adapted from \citet{elbalghiti2023generalization}]\label{def:distance_to_degeneracy}
			For a prediction $\hat c\in\bbR^d$, define its \emph{distance to degeneracy} by
			\[
			d_{\mathrm{deg}}(\hat c)
			:= \inf\bigl\{\|\Delta\|:\ w^*(\hat c+\Delta)\neq w^*(\hat c)\bigr\}.
			\] In particular, $d_{\mathrm{deg}}(\hat c)=0$ whenever $\hat c$ lies on a cone boundary.
		\end{definition}
		
		When $S$ is polyhedral, this quantity admits an explicit formula. Suppose $w^*(\hat c)=w_{(j)}$ is unique. Since
		\[
		N_j = \bigcap_{k\neq j}\{u\in\bbR^d: u^\top(w_{(k)}-w_{(j)})\ge 0\},
		\]
		the distance from $\hat c$ to the boundary of $N_j$ is
		\begin{equation}\label{eq:distance_to_degeneracy_formula}
			d_{\mathrm{deg}}(\hat c)
			= \min_{k\neq j}
			\frac{\hat c^\top(w_{(k)}-w_{(j)})}{\|w_{(k)}-w_{(j)}\|_*},
		\end{equation}
		where $\|\cdot\|_*$ is the dual norm of $\|\cdot\|$. Under the Euclidean norm, \eqref{eq:distance_to_degeneracy_formula} reduces to the usual perpendicular distance from $\hat c$ to the nearest supporting hyperplane of the cone.
		
		The right panel of Figure~\ref{fig:margin_based} illustrates the distance to degeneracy for the yellow prediction vector. In this example, the distance to degeneracy is given by the length of the red vector.
		\begin{figure}[ht]
			\centering
			\includegraphics[width=\linewidth]{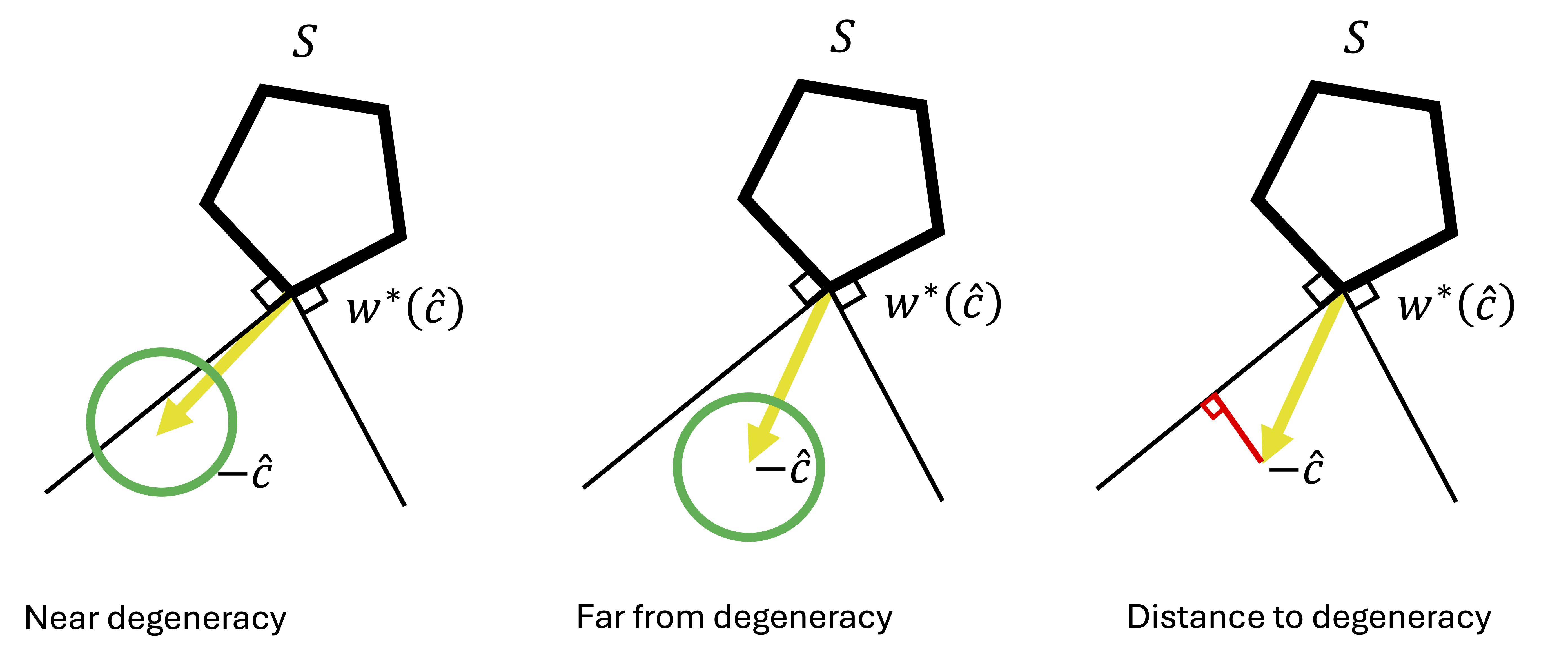}
			\caption{Illustration of the margin-based method. Left and middle: Two predictions may have confidence regions of the same size, yet only the prediction near a cone boundary is decision-uncertain. Right: under the Euclidean norm, the distance to degeneracy is the perpendicular distance from the predicted cost direction to the nearest cone boundary.}
			\label{fig:margin_based}
		\end{figure}
		
		Definition~\ref{def:distance_to_degeneracy} yields a simple certification principle. If the entire confidence region $\mathcal{B}_t(x)$ lies inside a single cone, then every plausible cost vector induces the same optimizer, so the decision at $x$ is already certified. A sufficient condition is
		\begin{equation}\label{eq:certification_condition}
			r_t(x) < d_{\mathrm{deg}}\bigl(h_t(x)\bigr).
		\end{equation}
		Whenever \eqref{eq:certification_condition} holds, further sampling at $x$ may improve prediction accuracy but cannot change the induced decision. It will result in a stopping time for collecting labels of $x$.
		Conversely, if the radius is comparable to or larger than the distance to degeneracy, then the current uncertainty region intersects a cone boundary, and additional labels at $x$ may still be decision-relevant.
		
		Despite this intuitive idea, the value of $r_t(x)$ for each feature point $x$ and each time step $t$ requires careful design, and depends on both the training loss and the structure of the prediction class. Intuitively, a simpler prediction class, such as a linear model, leads to a faster estimation error rate and therefore a smaller confidence radius. Because adaptive data selection produces non-i.i.d.\ observations, establishing the convergence of the prediction model requires additional techniques, such as reweighting. We refer readers to \citet{liu2023active} for a specific setup and detailed analysis.
		
		The effectiveness of this margin-based approach depends on how often the data-generating distribution places mass near cone boundaries. A standard way to quantify this is through a soft-margin condition.
		
		\begin{definition}[Soft-margin condition, \citet{liu2023active}]\label{def:soft_margin}
			We say that the contextual distribution satisfies a \emph{soft-margin condition} with parameters $(K,\kappa)$ if, for all sufficiently small $t>0$,
			\[
			\mathbb{P}_X\bigl(d_{\mathrm{deg}}(\bE[c|X])\le t\bigr)
			\le K t^\kappa.
			\]
		\end{definition}
		
		The parameter $\kappa$ measures how much probability mass lies near degeneracy. A large value of $\kappa$ means that most contexts are well separated from cone boundaries, in which case many decisions can be certified using only a small number of labels. This is the analogue of Tsybakov's margin condition in classification. Under such conditions, margin-based active-learning methods can enjoy substantially faster rates than passive data collection, because only near-boundary contexts require intensive sampling; see \citet{liu2023active}.
		
		Two observations help justify the soft-margin condition in practice. First, this condition is governed by the conditional mean $\mathbb{E}[c\mid X]$, rather than by the entire conditional distribution of $c\mid X$. Thus, even if the distribution of $c\mid X$ is continuous and places substantial density near the degeneracy set, its conditional mean may still be well separated from degeneracy. Second, any distribution of \(\bE[c \mid X]\) can be made to satisfy the soft-margin condition by appropriately stretching the vectors \(\bE[c \mid X]\); see, for example, Example 4 in \citet{liu2023active}. 
		
		
		\subsection{Direction-based data collection}\label{sec:direction}
		
		The direction-based approach starts from the scale-invariant property in Theorem \ref{1.2}. This property implies that the scale information of the prediction is decision-irrelevant, while small directional errors can move the prediction across a cone boundary and induce positive regret. This suggests that, for data collection, uncertainty should be measured through disagreement in \emph{direction} rather than disagreement in Euclidean norm. This insight can be used to design sequential experiments for data collection.
		In sequential experimental design, a common strategy is to query the design whose predictive distribution exhibits the largest uncertainty; see, for example, \citet{zhao2024experimental}.
		
		A natural, but decision-blind, choice is to measure uncertainty through the $\ell_2$ spread of the predictions produced by the current candidate models. The following example shows why this can be misleading for DFL.
		
		\begin{figure}[H]
			\centering
			\begin{minipage}[t]{0.40\textwidth}
				\centering
				\includegraphics[width=\linewidth]{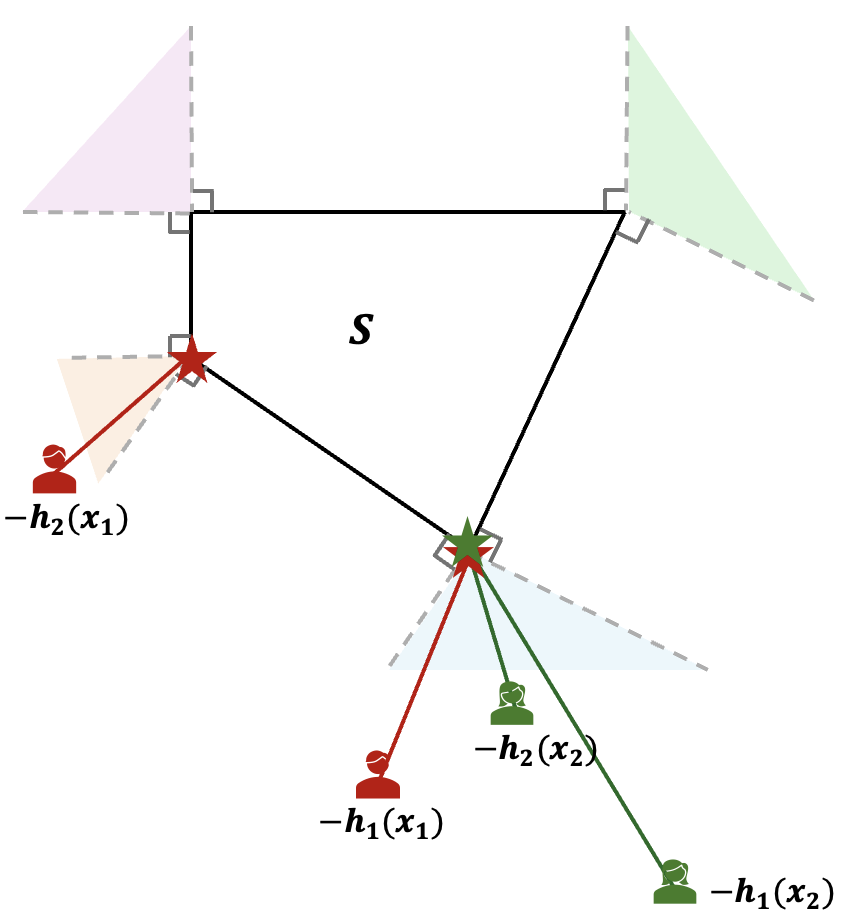}
				\caption{Predicted cost vectors and their corresponding optimizers.}
				\label{fig:predicted_cost_vector_with_corresponding_optimizer}
			\end{minipage}
			\hfill
			\begin{minipage}[t]{0.48\textwidth}
				\centering
				\includegraphics[width=\linewidth]{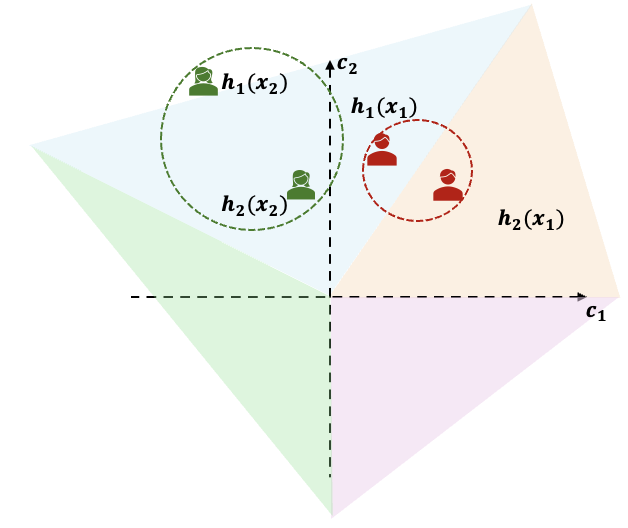}
				\caption{The same predictions viewed through Euclidean spread.}
				\label{fig:predicted_cost_vector_with_l2_distance}
			\end{minipage}
		\end{figure}
		
		\begin{example}\label{example:motivation}
			Consider a two-dimensional setting with two candidate predictors $\mathcal{H}=\{h_1,h_2\}$ and two candidate designs $\mathcal{X}=\{x_1,x_2\}$. In Figure~\ref{fig:predicted_cost_vector_with_corresponding_optimizer}, red points correspond to predictions at $x_1$ and green points correspond to predictions at $x_2$. For ease of visualization, each cost vector is represented by its negative direction. Although the two green predictions are farther apart in Euclidean distance, they lie in the same cone and therefore induce the same optimizer, as shown in Figure \ref{fig:predicted_cost_vector_with_l2_distance}. By contrast, the two red predictions are closer in Euclidean distance but lie in different cones, so they induce different optimizers in Figure \ref{fig:predicted_cost_vector_with_l2_distance}. Hence, $x_1$ is more uncertain from the perspective of decision loss, even though $x_2$ has a larger $\ell_2$ spread.
		\end{example}
		
		Example~\ref{example:motivation} shows that the geometry relevant to DFL is angular rather than radial. Motivated by this observation, for a given design $x$ and a class of candidate models $H$, consider the following three uncertainty scores:
		\begin{enumerate}[
			label=\textbf{Metric \arabic*.},
			leftmargin=3.2em,
			align=left
			]
			\item Standard Euclidean disagreement: $
			U^{\ell_2}(x)
			:= \max_{h_1,h_2\in H}\|h_1(x)-h_2(x)\|$,

			\item Computational intractable decision-loss disagreement: $$
			U^{\mathrm{decision}}(x)
			:= \max_{h_1,h_2\in H}
			\max\!\bigl\{\ld(h_1(x),h_2(x)),\,\ld(h_2(x),h_1(x))\bigr\}.$$
			
			\item Normalized directional disagreement: $U^{\mathrm{dir}}(x)
			:= \max_{h_1,h_2\in H}
			\left\|
			\frac{h_1(x)}{\|h_1(x)\|_2}
			-
			\frac{h_2(x)}{\|h_2(x)\|_2}
			\right\|_2$.
		\end{enumerate}
		
		Metric~3 can also be viewed as an angular disagreement measure since $
		\left\|\frac{u}{\|u\|_2}-\frac{v}{\|v\|_2}\right\|_2
		=2\sin\!\left(\frac{\angle(u,v)}{2}\right)$.

		\begin{figure}[h]
			\centering
			\begin{minipage}[t]{0.40\textwidth}
				\centering
				\includegraphics[width=\linewidth]{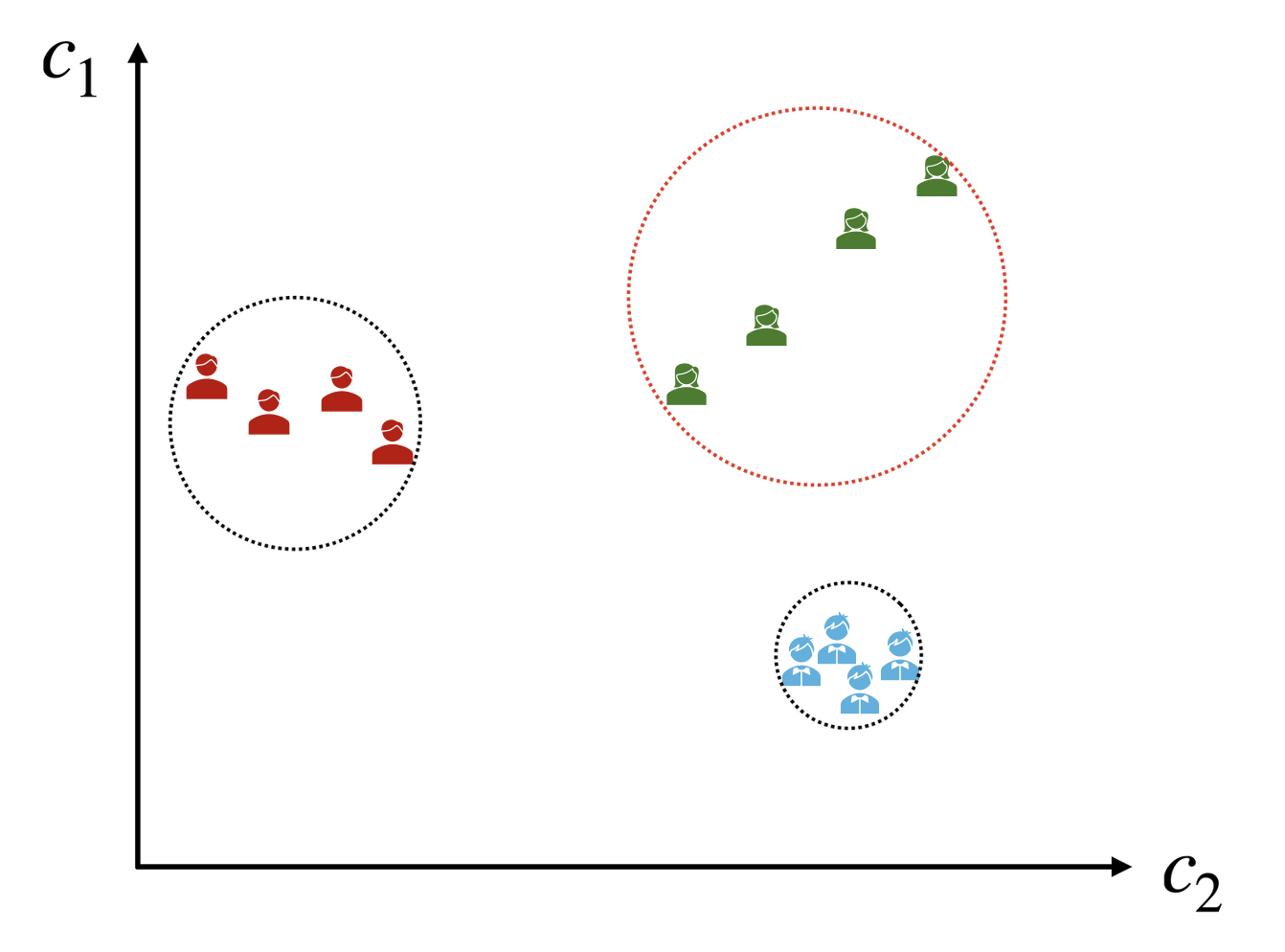}
				\caption{\textbf{Metric 1: }$\ell_2$-based uncertainty.}
				\label{fig:l2_uncertainty}
			\end{minipage}
			\hfill
			\begin{minipage}[t]{0.50\textwidth}
				\centering
				\includegraphics[width=\linewidth]{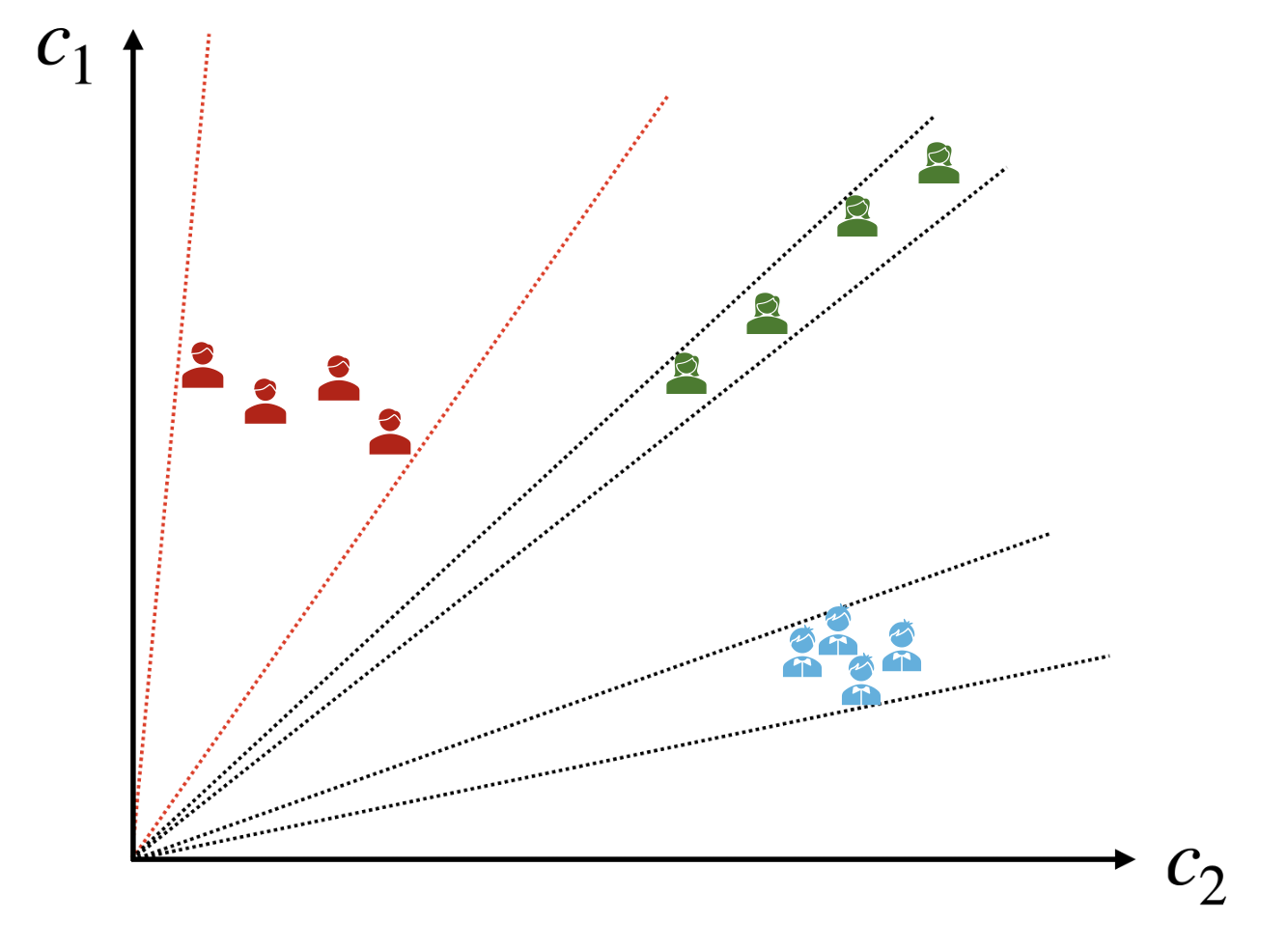}
				\caption{\textbf{Metric 3: } Direction-based uncertainty.}
				\label{fig:angle_uncertainty}
			\end{minipage}
		\end{figure}
		
		Figures~\ref{fig:l2_uncertainty} and \ref{fig:angle_uncertainty} provide a second illustration of these uncertainty metrics. There are three candidate designs (red, green, and blue) and four candidate predictors $h_1,\dots,h_4\in\mathcal{H}$. Under the $\ell_2$ metric, the green design has the largest spread and would be selected by a standard uncertainty-sampling rule. Under the directional metric, however, the red design has the largest angular disagreement and is the more decision-relevant query. This reflects the fact that the green predictions mainly differ in magnitude, whereas the red predictions differ in direction and are therefore more likely to cross cone boundaries.
		
		Metric~1 is computationally cheap but decision blind. Metric~2 is closely aligned with the decision loss, but evaluating it requires repeated solution of the downstream optimization problem for many pairs of predictions, which can be prohibitive in sequential settings. Metric~3 sits between these extremes: it is computationally lightweight, requires no optimization-oracle calls, and directly exploits the scale invariance of the decision loss.
		
		Based on this directional uncertainty metric, \textbf{Metric 3}, \citet{wan2026directional} propose a new decision-focused sequential experimental design approach. At each iteration, candidate designs are sampled randomly with probabilities proportional to their directional uncertainty, and the hypothesis class is shrunk appropriately over time. They show that this approach leads to an earlier stopping time and a faster convergence rate. Its theoretical advantage over decision-blind design is further analyzed in \citet{wan2026directional} under several special noise distributions. They also demonstrate strong empirical performance on real-data experiments.
		
		\paragraph{Applications of decision-focused data collection methods.}
		Decision-focused data collection is useful whenever costly labels are only intermediate objects and the ultimate objective is to improve downstream actions. In marketing and assortment planning, it prioritizes customer segments for which plausible demand models lead to different stocking or pricing decisions; see, e.g., \citet{liu2023value}. In personalized treatment, it targets patient groups whose response uncertainty could change the selected intervention; see, e.g., \citet{chung2026improving}. In A/B testing for service design, it guides the selection of experimental designs that are most informative for estimating how short-term effects translate into long-term outcomes, thereby supporting better product or service decisions; see, e.g., \citet{chernozhukov2024applied}. The same principle also appears in modern AI systems: when preference labels or human reward evaluations are expensive, one should prioritize prompts for which plausible models disagree about the induced action, rather than merely about the scale of a score; see, e.g., \citet{wan2026directional}.

		\section{Decision-Focused Distance Between Distributions}\label{sec:distance}
		
		Distances between probability distributions are fundamental tools in statistics and machine learning. They are used for clustering, interpolation, kernel weighting, domain adaptation, and distribution-shift analysis. In DFL, however, a good distance should not only capture geometric discrepancy in the raw outcome space; it should also reflect whether two distributions induce similar \emph{decisions}. This section discusses why standard distances can be inadequate for that purpose and introduces a coupling-based decision-focused alternative, proposed in \citet{liu2026decision}.
		
		\subsection{Why standard distances may be inadequate}
		
		Classical distances compare distributions pointwise. Let $P$ and $Q$ be probability measures on a measurable space $(\Omega,\mathcal{F})$ with densities $p$ and $q$ with respect to a common dominating measure $\mu$. The total-variation distance is
		\[
		d_{\mathrm{TV}}(P,Q)
		:= \sup_{A\in\mathcal{F}} |P(A)-Q(A)|
		= \frac12\int_{\Omega}|p(x)-q(x)|\,d\mu(x),
		\]
		and the Kullback--Leibler divergence is
		\[
		D_{\mathrm{KL}}(P\|Q)
		:= \int_{\Omega} p(x)\log\!\left(\frac{p(x)}{q(x)}\right)d\mu(x),
		\]
		with the usual convention that $D_{\mathrm{KL}}(P\|Q)=+\infty$ when $P$ is not absolutely continuous with respect to $Q$.
		
		A limitation of these divergences is that they compare mass \emph{at the same location}. When supports do not overlap, the KL divergence may become infinite, and total variation may saturate, even if the two distributions are close in a geometric sense. This motivates transport-based distances such as the $p$-Wasserstein distance,
		\begin{equation}\label{eq:wp_def}
			W_p(P,Q)
			:= \left(
			\inf_{\pi\in\Pi(P,Q)}
			\int_{\Omega\times\Omega}\|x-y\|^p\,d\pi(x,y)
			\right)^{1/p},
		\end{equation}
		where $\Pi(P,Q)$ denotes the set of couplings of $P$ and $Q$.
		
		Wasserstein distance incorporates geometry by asking how far mass must be transported to transform one distribution into the other. But it is still \emph{decision blind}. In linear optimization, moving a cost vector within the same normal cone may have no effect on the optimizer at all, whereas a tiny perturbation across a cone boundary may change the decision discontinuously. From a decision-making perspective, these two perturbations should be treated very differently, even if their Euclidean size is reversed.
		
		\begin{figure}[H]
			\centering
			\includegraphics[width = 0.48\linewidth]{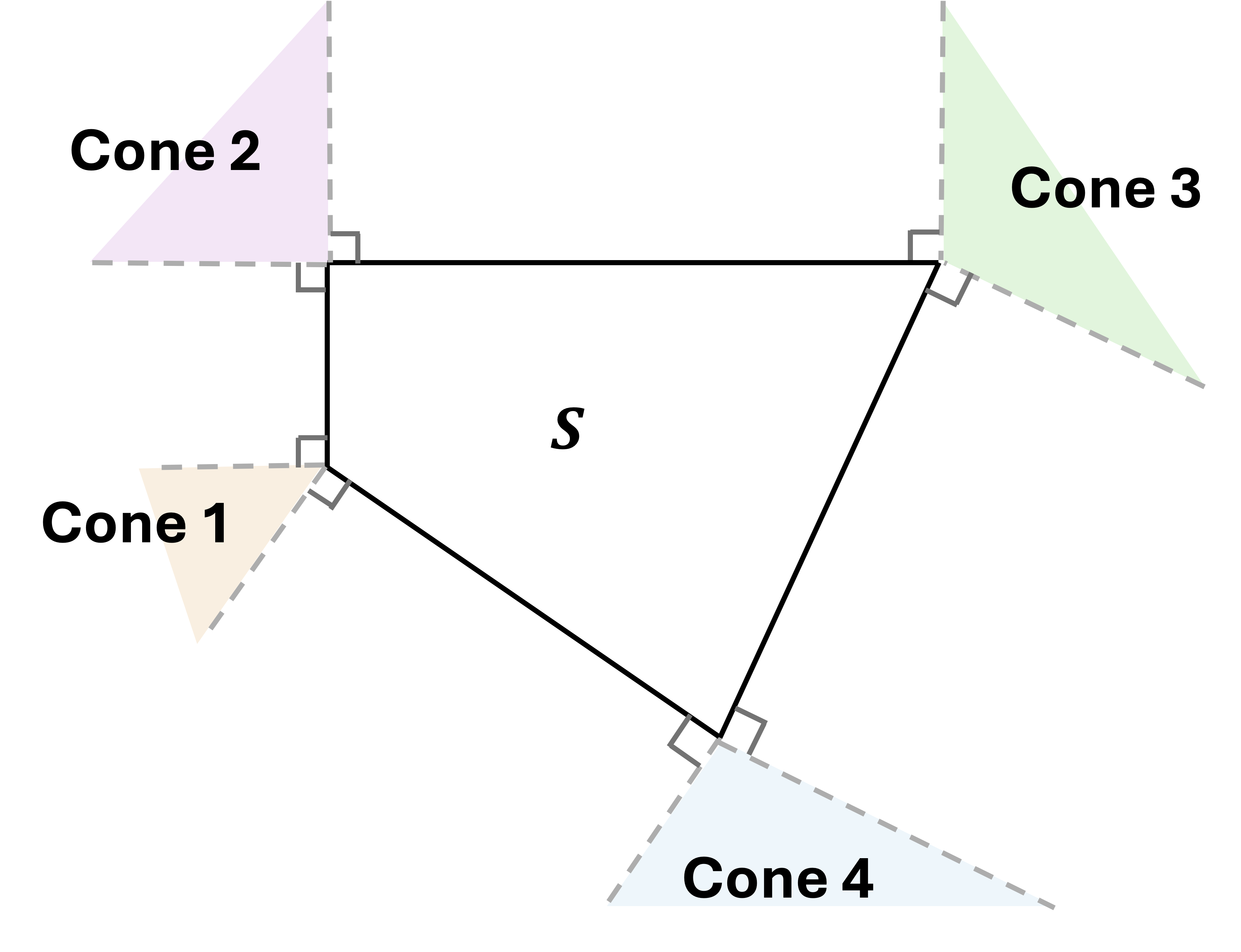}
			\includegraphics[width=0.48\linewidth]{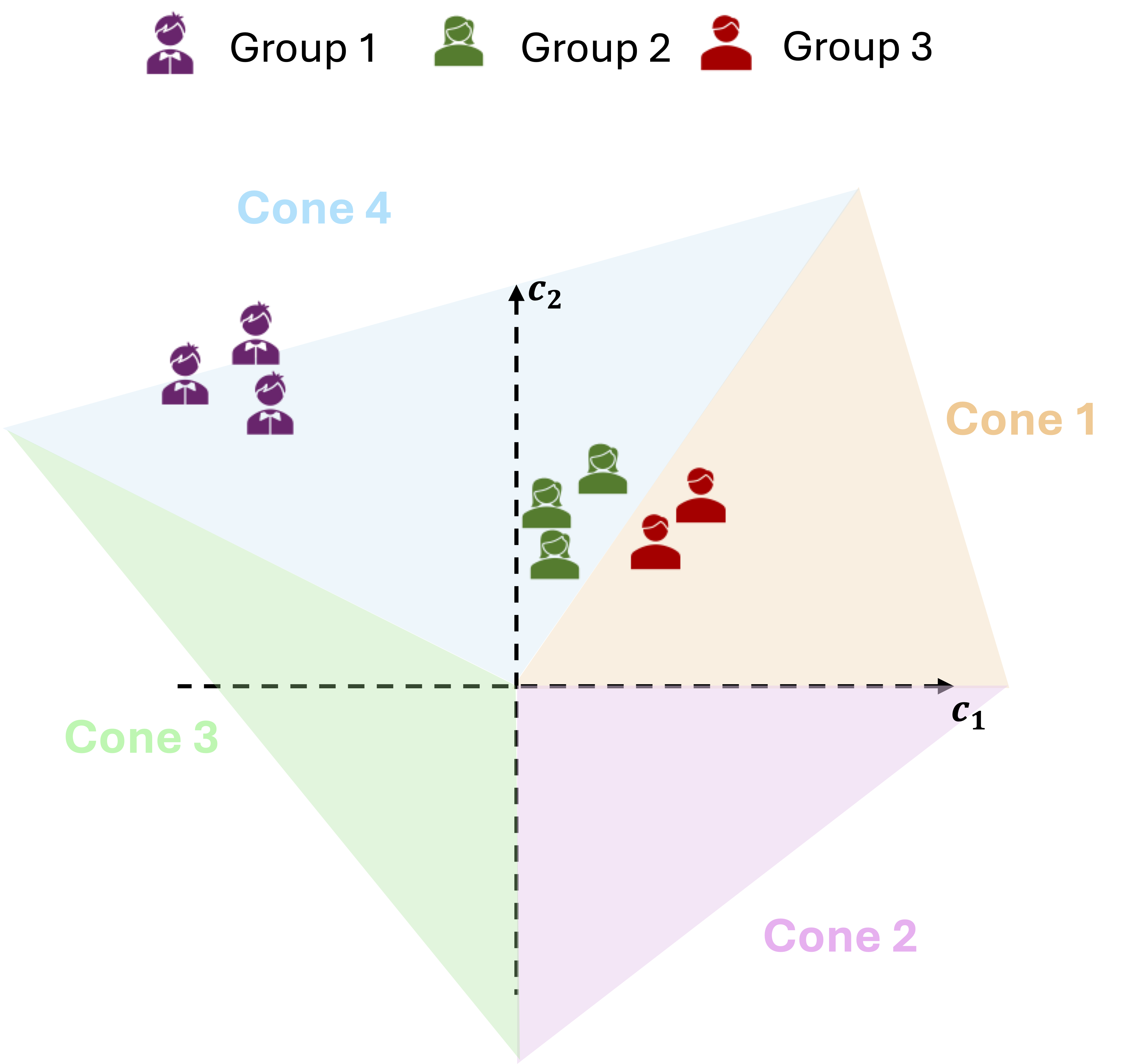}
			\caption{\textbf{Geometric similarity does not imply decision similarity.} \textbf{Left:} A polyhedral feasible region (S) with four extreme points and their associated normal cones. \textbf{Right:} Three example groups of cost vectors. Groups~1 and~2 lie in the same normal cone and are decision-equivalent, whereas Group~3 lies in a different normal cone and induces a different optimizer.}
			\label{fig:intro_customer}
		\end{figure}
		
		Figure~\ref{fig:intro_customer} illustrates this geometry in two dimensions. Let $S\subset\bbR^2$ be a polyhedral feasible region with extreme points $w_1,\dots,w_4$. The cost space can then be partitioned into four regions according to which extreme point is optimal, as shown in the right panel of Figure~\ref{fig:intro_customer}. Consequently, if two cost vectors $x$ and $y$ belong to the same cone, then $w^*(x)=w^*(y)$ and the decision loss $\ld(x,y)$ is zero. A decision-focused notion of distance should therefore regard distributions supported in the same cone as close---or even identical from the standpoint of optimization---regardless of their Euclidean separation.
		
		This observation already reveals a fundamental distinction from Wasserstein geometry. Groups~1 and~2 in Figure~\ref{fig:intro_customer} are geometrically separated but decision-equivalent because they lie in the same normal cone. By contrast, Groups~2 and~3 may be geometrically close, yet they induce different decisions because they lie in different normal cones. The use of downstream decision cost to evaluate distributional distance was first studied by \citet{bertsimas2023optimization}, whose focus is scenario reduction for the SAA method. The idea of an optimal-transport-based decision-focused divergence was later considered by \citet{rodriguez2024right}, primarily from a computational perspective. This divergence was subsequently formalized and rigorously analyzed for linear programs by \citet{liu2026decision}, where it is called the \textit{decision-focused optimistic distance}.
		
		\subsection{Decision-focused optimistic distance}
		
		Let $\mu$ and $\nu$ be probability measures on $\bbR^d$ with finite first moments, and let $\Gamma(\mu,\nu)$ denote the set of all couplings of $\mu$ and $\nu$. For any coupling $\gamma\in\Gamma(\mu,\nu)$, define the \emph{decision-focused divergence}
		\begin{equation}\label{eq:df_divergence}
			\wspo(\mu,\nu;\gamma)
			:= \int_{\bbR^d\times\bbR^d}\ld(x,y)\,\gamma(dx,dy).
		\end{equation}
		Because $\ld(x,y)\le \operatorname{diam}(S)\|x-y\|_*$, the integral in \eqref{eq:df_divergence} is well defined whenever $\mu$ and $\nu$ have finite first moments.
		
		The quantity \eqref{eq:df_divergence} depends not only on the marginals but also on the coupling. This dependence is meaningful. If $X\sim\mu$ and $Y\sim\nu$ are independent, then $\gamma=\mu\otimes\nu$ and \eqref{eq:df_divergence} reduces to the familiar expected regret
		\begin{equation}\label{eq:independent_regret}
			R(\mu,\nu)
			:= \bE\bigl[\ld(X,Y)\bigr]
			= \wspo(\mu,\nu;\mu\otimes\nu).
		\end{equation}
		However, in many applications, independence is unrealistic. If the two distributions describe the same unit observed under different conditions---for instance, the same patient at different ages or the same market before and after a shift---then the way mass is aligned across the two distributions matters. Ignoring this dependence can substantially overestimate or underestimate average decision loss.
		
		This motivates the following definition of decision-focused divergence via optimal transport.
		
		\begin{definition}[Decision-focused optimistic divergence]\label{def:df_dist}
			For probability measures $\mu$ and $\nu$ on $\bbR^d$, define
			\begin{align}
				\wdfo(\mu,\nu)
				&:= \inf_{\gamma\in\Gamma(\mu,\nu)} \wspo(\mu,\nu;\gamma).
				\label{eq:wdf_optimistic}
			\end{align}
		\end{definition}
		This divergence, proposed by Suhan Liu and Mo Liu, is called the \emph{decision-focused optimistic divergence} because it is defined through the coupling that minimizes the decision loss between two distributions. In some OR settings, the maximum decision loss is also of interest as a pessimistic assessment of risk. For this reason, \citet{liu2026decision} also define the \emph{decision-focused pessimistic divergence} by
		\[
		\wdfr(\mu,\nu)
		:= \sup_{\gamma\in\Gamma(\mu,\nu)} \wspo(\mu,\nu;\gamma).
		\]
		Since the independent coupling $\mu\otimes\nu$ belongs to $\Gamma(\mu,\nu)$, we always have
		\[
		0\le \wdfo(\mu,\nu)
		\le R(\mu,\nu)
		\le \wdfr(\mu,\nu).
		\]
		The optimistic version gives the most favorable alignment between the two distributions, whereas the pessimistic version gives the least favorable one. Both quantities are generally asymmetric, because evaluating decisions induced by $\mu$ under $\nu$ is different from evaluating decisions induced by $\nu$ under $\mu$.
		
		When symmetry is desirable, one may use standard symmetrizations such as
		\[
		W^O_{\mathrm{sym}}(\mu,\nu)
		:=\wdfo(\mu,\nu)+\wdfo(\nu,\mu),
		\qquad
		W^R_{\mathrm{sym}}(\mu,\nu)
		:=\wdfr(\mu,\nu)+\wdfr(\nu,\mu),
		\]
		or Jensen--Shannon-type variants based on the midpoint mixture $m:=(\mu+\nu)/2$. In many applications, however, the asymmetric form is precisely the quantity of interest because it corresponds to evaluating one distribution through decisions trained on another.
		\citet{liu2026decision} show that the \textit{decision-focused optimistic divergence} $\wdfo(\cdot,\cdot)$ has the following favorable properties:
		
		\begin{enumerate}
			\item \textbf{Computational efficiency.} The quantity \(\wdfo(\mu,\nu)\) can be computed efficiently through a transformation of the \(W_2\) distance between \((\bar w^*)_{\#}\mu\) and \(\nu\), where \((\bar w^*)_{\#}\mu\) denotes the pushforward of \(\mu\) onto the set of optimal extreme points.
			
			\item \textbf{Statistical efficiency.} The estimation error rate for the \textit{decision-focused optimistic divergence} is independent of the dimension of the measure space. By contrast, for many classical distributional distances, the estimation error typically decays at the rate \(O(n^{-1/d})\), where \(d\) is the dimension of the measure space. The \textit{decision-focused optimistic divergence} avoids this curse of dimensionality by exploiting the finite set of extreme points in linear programming.
			
			\item \textbf{Control by classical metrics.} The quantity \(\wdfo(\mu,\nu)\) is Lipschitz continuous with respect to total variation distance and the \(W_1\) distance, and is further controlled by the \(W_2\) distance after pushforward to the decision space.
		\end{enumerate}
		
		\citet{liu2026decision} also consider a regularized version of the \textit{decision-focused optimistic divergence} by adding a KL-divergence penalty to the coupling problem. The weight of this penalty controls the degree of pessimism in the resulting coupling. See Proposition~4 of \citet{liu2026decision} for details.

		\subsection{Applications to interpolation and clustering}

		In the DFL setting, the \textit{decision-focused optimistic divergence} can naturally replace the Wasserstein 2 distance in various optimal transport problems; see, e.g., \citet{cheng2026generative}. The resulting decision-focused coupling has a wide range of potential applications in DFL. Here, we use clustering and interpolation as two illustrative examples.
		
		\paragraph{Decision-aware clustering.}
		Similarity between coefficient distributions does not necessarily imply similarity between the induced optimal decisions. Therefore, when clustering samples, for example, to provide the same service to similar customers and reduce the cost of personalization (or customization), the clustering criterion should account for downstream decisions rather than distributional similarity alone.

		We use the newsvendor problem as a motivating example. Suppose each observational unit is represented not by a single point, but by an entire distribution, such as a customer-specific demand distribution, a patient-specific response distribution, or a city-level sales distribution. Standard clustering methods based on Euclidean or Wasserstein distance group samples that are geometrically similar. In contrast, decision-focused clustering groups samples according to the similarity of their induced optimal actions.
		
		\begin{figure}[t]
			\centering
			\includegraphics[width=0.5\linewidth]{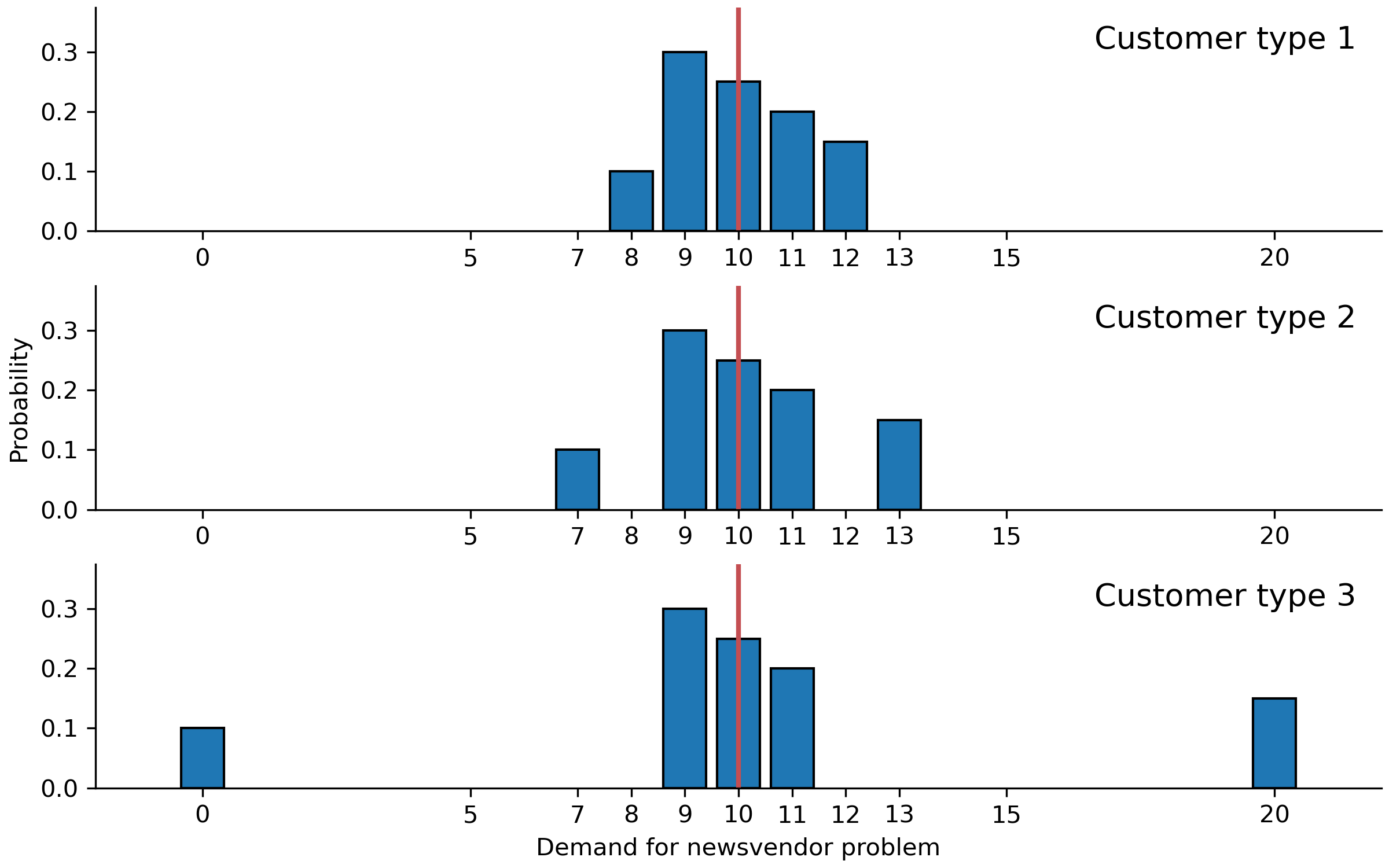}
			\caption{Three customer types in a newsvendor problem. The red line marks the same decision-optimal quantile for all three distributions. Although the distributions differ in shape and tail behavior, they induce the same optimal order quantity, equal to 10.}
			\label{fig:newsvendor_same_quantile}
		\end{figure}
		
		Figure~\ref{fig:newsvendor_same_quantile} presents a simple newsvendor example. The three customer types have visibly different demand distributions, and standard distributional distances would treat them as distinct. However, suppose their unit selling price $p$ and unit procurement cost $c$ are the same across the three customer types. Then all three distributions share the same decision-relevant quantile, \(F^{-1}\!\left(\frac{p-c}{p}\right)\). In particular, when $p$ and $c$ are chosen such that $(p-c)/p=0.7$, this corresponds to the $70\%$ quantile, and the three customer types yield the same optimal order quantity, namely 10. From a decision-focused perspective, these distributions should be regarded as similar. This is exactly the type of structure that a decision-focused distance is designed to preserve.

		This viewpoint leads naturally to decision-aware barycenters and clustering rules. For example, given distributions $\mu_1,\dots,\mu_n$ and kernel weights $\omega_i(x)$ centered at a target covariate value $x$, one may define a decision-aware interpolant through
		\[
		\hat\mu_x\in\argmin_{\nu}\sum_{i=1}^n \omega_i(x)\,W^O_{\mathrm{sym}}(\mu_i,\nu),
		\]
		while a decision-aware clustering rule can be based on medoids or barycenters that minimize within-cluster DF distance. The precise optimization problem depends on the application, but the common principle is simple: group or interpolate distributions according to \emph{decision similarity}, not merely raw geometric proximity.
		
		\paragraph{Decision-aware interpolation.}
		Interpolation poses a related challenge. Suppose distributions are indexed by a covariate such as income, age, or time, and we wish to estimate the distribution corresponding to an intermediate covariate value. A naive strategy averages empirical histograms pointwise. When supports are separated, however, this can create unrealistic intermediate distributions.
		
		\begin{figure}[htbp]
			\centering
			\includegraphics[width=0.5\linewidth]{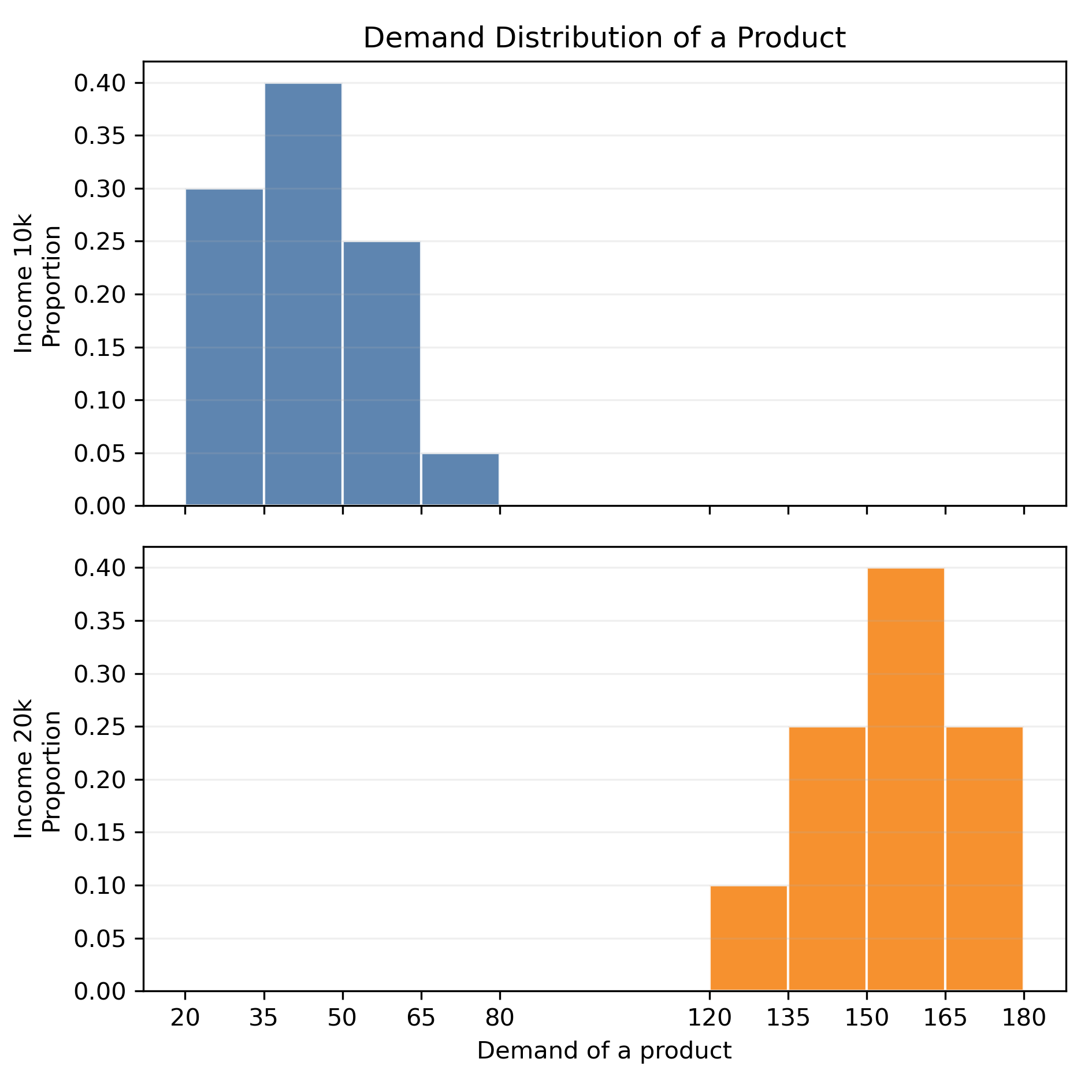}
			\caption{Demand distributions of the same product in two endpoint markets. The supports are largely separated, suggesting that interpolation should move mass rather than simply average pointwise frequencies.}
			\label{fig:demand_distribution_income_comparison}
		\end{figure}
		
		\begin{figure}[ht]
			\centering
			\includegraphics[width=0.48\linewidth]{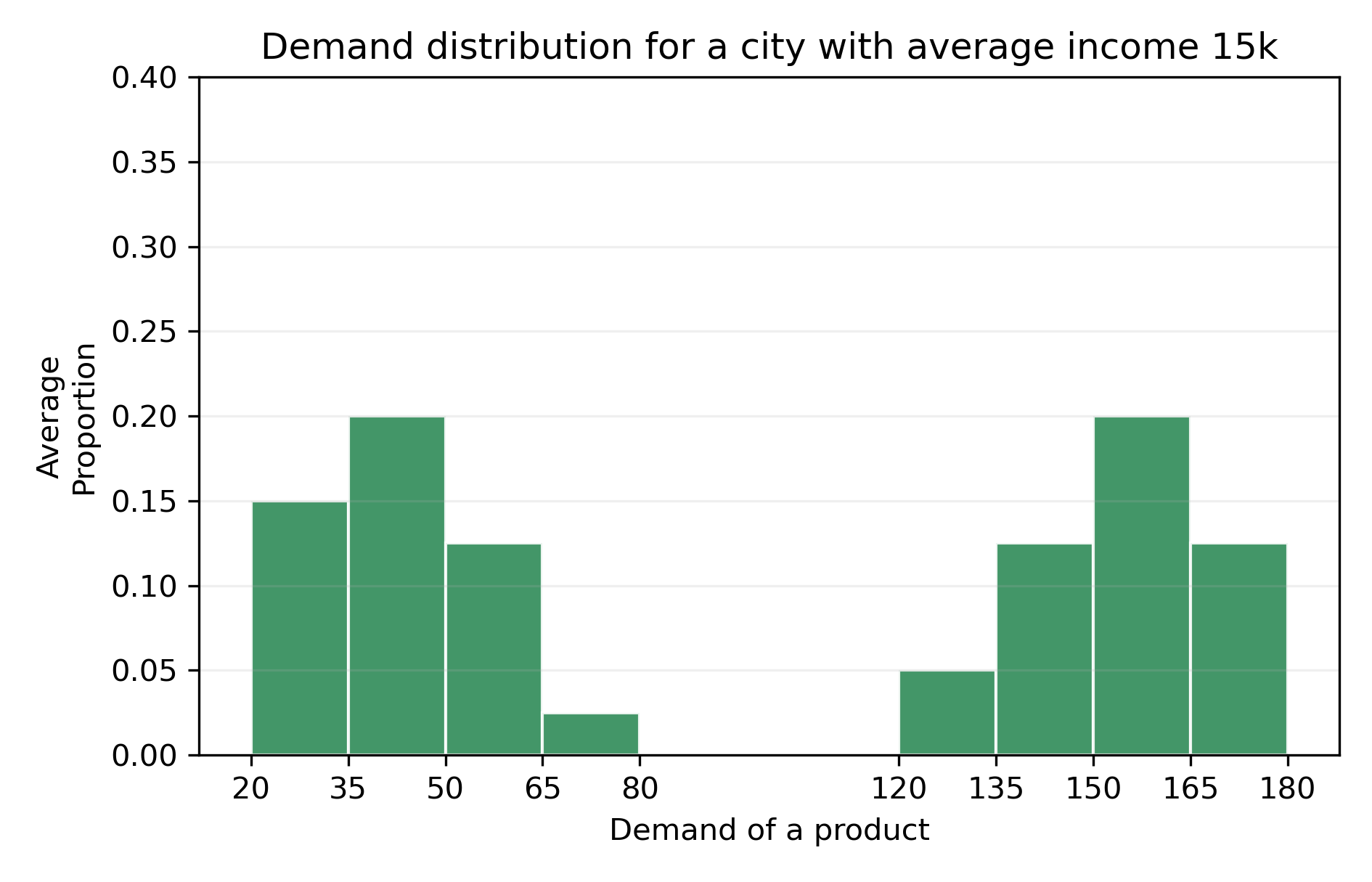}
			\includegraphics[width=0.48\linewidth]{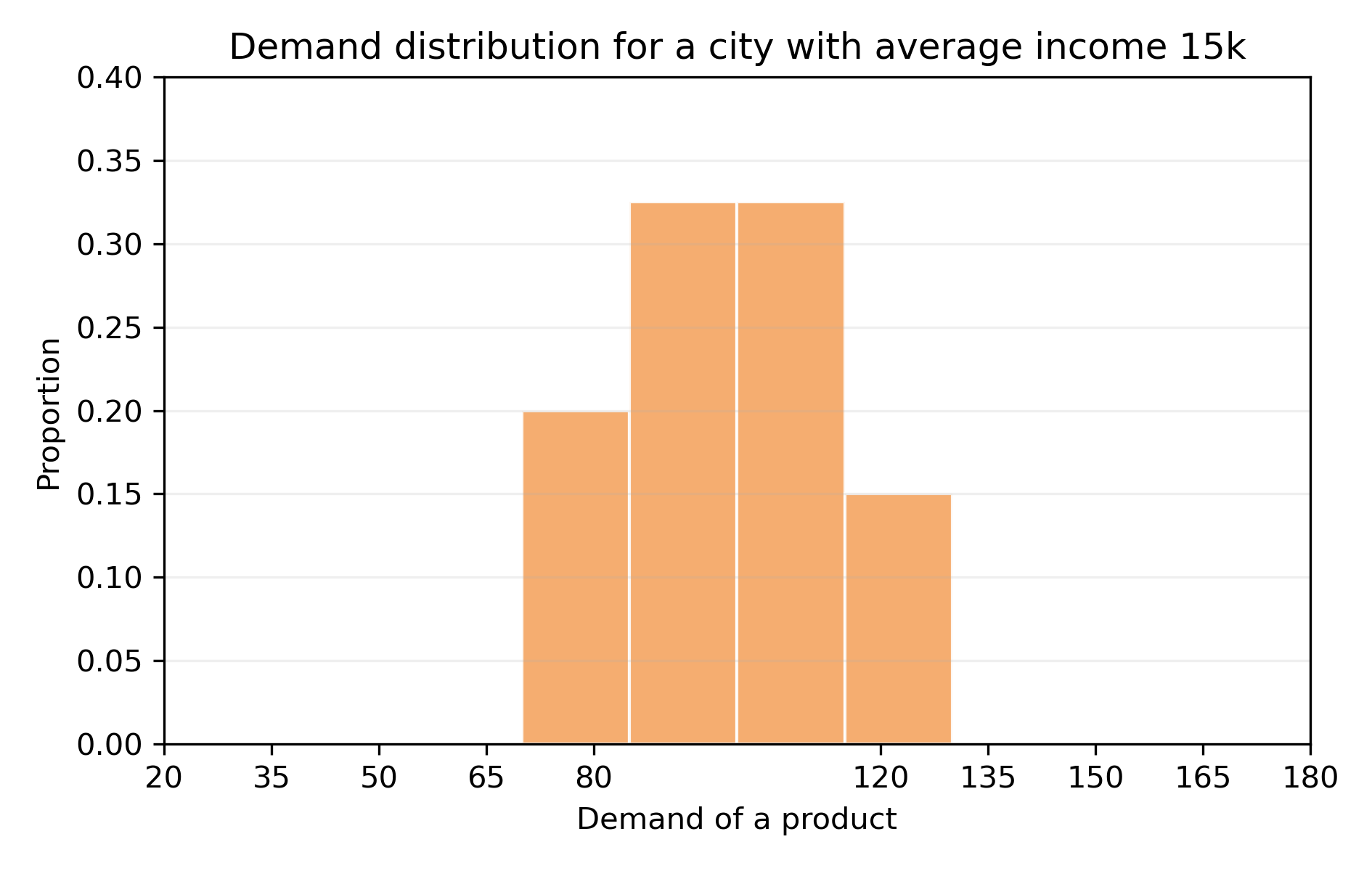}
			\caption{Two candidate interpolations for an intermediate market. \textbf{Left:} naive averaging of endpoint histograms produces an artificial bimodal distribution. \textbf{Right:} a transport- or decision-aware interpolation yields a smoother intermediate distribution that better reflects a gradual demand shift.}
			\label{fig:distribution_interpolation}
		\end{figure}
		
		Figures~\ref{fig:demand_distribution_income_comparison} and \ref{fig:distribution_interpolation} illustrate the point. When the demand distributions in low- and high-income markets have disjoint or weakly overlapping supports, pointwise averaging yields the bimodal shape in Figure~\ref{fig:distribution_interpolation} (left), which is often not operationally plausible. In contrast, transport-based interpolation moves mass from one region to another and produces the unimodal intermediate distribution in Figure~\ref{fig:distribution_interpolation} (right). A decision-focused interpolation goes one step further: it favors interpolants that preserve the downstream action, such as relevant quantiles, optimal stocking decisions, or optimizer cones.

		Interpolation induced by optimal transport is often referred to as McCann interpolation in statistical learning. In the DFL setting, \citet{liu2026decision} apply this idea using the \textit{decision-focused optimistic divergence} and call the resulting construction \emph{decision-focused interpolation}. Empirical results on real datasets in \citet{liu2026decision} show that this decision-focused interpolation produces distribution forecasts with lower decision cost than decision-blind interpolation methods.
		
		\section{Open Statistical Learning Questions in DFL}
		
		Tailoring traditional statistical learning methods to DFL remains an important and active research direction. Existing work has already begun to address several foundational questions, including statistical learning rates \citet{elbalghiti2023generalization,hu2022fast}, robust optimization \citet{im2025smart}, learning under bandit feedback \citet{hu2025contextual}, data collection \citet{liu2023active,bennouna2025data}, sequential experimental design \citet{wan2026directional}, uncertainty quantification \citet{yeh2025conformal}, and distributional distance quantification \citet{liu2026decision}. These works suggest that many classical learning questions can be revisited through a decision-focused lens, often leading to new phenomena that do not appear in prediction-focused learning.
		
		At the same time, a broad range of statistical learning problems remains largely unexplored in DFL. Examples include Value-at-Risk (VaR) estimation, offline policy evaluation, denoising, missing-feature completion, clustering, and feature selection. In many of these settings, the key challenge is to identify which aspects of uncertainty matter for downstream decisions and which do not. More generally, an important question is whether one can develop decision-aware versions of classical statistical procedures that retain their computational tractability while achieving stronger downstream guarantees. Another appealing direction is to identify conditions under which faster learning rates can be achieved, for example, under special noise distributions, margin conditions, or structural assumptions on the hypothesis class.
		
		A second important direction is to broaden the range of real-world applications. DFL is naturally motivated by problems in service system design, pricing, personalized recommendation, delivery and logistics, energy systems, healthcare, and hospital operations. Many of these applications involve nonlinear objectives, uncertain constraints, or multi-stage uncertainty, and thus fall outside the clean stochastic linear optimization framework emphasized in this tutorial. Nevertheless, when suitable approximations or reformulations reduce these problems to linear or locally linear models, the tools developed here may still provide useful insights. For example, \citet{liu2023value} studies assortment optimization through such a perspective. 
		
		More broadly, going beyond stochastic linear optimization remains a central challenge for the field. Linear optimization provides a particularly tractable setting because the conditional mean is decision sufficient and the geometry of the oracle map is relatively explicit. For nonlinear objectives, however, the relevant decision statistic may be much more complicated, and in many problems, no low-dimensional point prediction is sufficient. Developing statistical theory for such settings, including uncertainty measures and notions of calibration, is an important frontier. 
		
		Finally, most existing analyses of DFL focus on static, single-period problems, partly because many basic statistical questions remain open even in this setting. Extending these ideas to multiperiod decision-making is a natural next step. Once temporal dependence and sequential feedback are introduced, DFL becomes closely connected to reinforcement learning, dynamic programming, and control; see, e.g., \citet{liu2022online,capitaine2026online}. This connection raises new questions about exploration, partial feedback, state uncertainty, and long run regret, while also creating opportunities to bring decision-focused ideas into broader sequential learning problems.

		\section{Conclusion}
		
		In this tutorial, we reviewed the main tools, challenges, and methods in DFL. We used data collection and distributional distance quantification as illustrative examples to show how traditional statistical learning tools can be adapted to the decision-focused setting. Overall, DFL remains at an early stage of development. The main message of this tutorial is that many classical tools from statistical learning continue to be useful, but they often must be redefined, reanalyzed, or redesigned once the ultimate goal is decision quality rather than prediction accuracy. Developing a more systematic understanding of this gap, both theoretically and computationally, remains one of the most promising directions for future research.
		%
		%
		%
		
		\ACKNOWLEDGMENT{Mo Liu gratefully acknowledges Erick Delage, Tito Homem-de-Mello, and Vishal Gupta for their valuable feedback. He also thanks the editors and four anonymous reviewers for their helpful comments and suggestions.}

		
		
		\bibliographystyle{informs2014} 
		\bibliography{reference} 

	\end{document}